%% file: paper.tex
\documentclass[]{fairmeta}
\usepackage{hyperref}
\hypersetup{
  pdftitle={Deliberate Practice},
  pdfkeywords={Artificial Intelligence,Machine Learning,Synthetic Data,Deliberate Practice,Active Learning,Adaptive Data Generation,Sample Efficiency,Scaling Laws,Data Curation,Entropy-Guided Sampling,Diffusion Models,Dataset Pruning,Hard Example Mining,Generalization,Efficient Training,Vision Transformers,Large-Scale Learning,Out-of-Distribution Generalization},
  pdfauthor={AskariHemmat et al.}
}
\usepackage{algorithm}
\usepackage{algcompatible}
\usepackage{algpseudocode}
\usepackage{amsmath, amssymb}

\usepackage{microtype}
\usepackage{graphicx}
\usepackage{subcaption}
\usepackage{booktabs} 
\usepackage{amssymb}
\usepackage{mathtools}
\usepackage{amsthm}
\usepackage{longtable}
\usepackage{sidecap}

\newcommand{\trace}{\operatorname{tr}}

\usepackage{amsthm,amsmath,amsfonts,thmtools}
\usepackage{thm-restate}
\declaretheorem[name=Theorem,style=examplestyle]{theorem}
\declaretheorem[name=Lemma,style=examplestyle]{lemma}
\declaretheorem[name=Corollary,style=examplestyle]{corollary}
\declaretheorem[name=Proposition,style=examplestyle]{proposition}
\declaretheorem[name=Definition,style=examplestyle]{definition}

\declaretheorem[name=Remark,style=examplestyle]{remark}

\usepackage{mdframed}
\usepackage{url}

\usepackage{graphicx}
\usepackage{wrapfig}
\usepackage{minitoc}
\usepackage{enumitem}
\usepackage{xcolor}
\usepackage{sidecap}
\sidecaptionvpos{figure}{t}
\usepackage{amsthm,amsfonts,thmtools}
\usepackage{thm-restate}

\renewcommand{\parallel}{{\mathrel{\vcenter{\offinterlineskip \hbox{$\scriptscriptstyle/$}\vskip-.5ex\hbox{$\scriptscriptstyle/$}}}}}

\author[1,*]{Reyhane Askari-Hemmat}
\author[1,*]{Mohammad Pezeshki}
\author[1, 2, 5]{Elvis Dohmatob}
\author[1]{Florian Bordes}
\author[1]{Pietro Astolfi}
\author[1]{Melissa Hall}
\author[1]{Jakob Verbeek}
\author[1]{Michal Drozdzal}
\author[1, 2, 3, 4]{Adriana Romero-Soriano}

\affiliation[1]{FAIR at Meta - Montreal, Paris, and New York City labs}
\affiliation[2]{Mila}
\affiliation[3]{McGill University}
\affiliation[4]{Canada CIFAR AI chair}
\affiliation[5]{Concordia University
}

\contribution[*]{Equal contribution.}

\abstract{
\input{0-abstract}
}

\date{\today}
\correspondence{\email{reyhaneaskari@meta.com, mpezeshki@meta.com}}

\title{Improving the Scaling Laws of Synthetic Data with Deliberate Practice}

\begin{document}

\maketitle

\input{1-intro}

\input{2-background}

\input{3-method}

\input{5-theory}
\input{4-experiments}
\input{6-related_work}

\input{7-discussion}

\clearpage
\newpage

\bibliographystyle{assets/plainnat}
\bibliography{paper}
\clearpage
\newpage
\beginappendix

\input{8-appendix}

\end{document}

%% file: 0-abstract.tex
Inspired by the principle of deliberate practice in human learning, we propose Deliberate Practice for Synthetic Data Generation (DP), a novel framework that improves sample efficiency through dynamic synthetic data generation. Prior work has shown that scaling synthetic data is inherently challenging, as naively adding new data leads to diminishing returns. To address this, pruning has been identified as a key mechanism for improving scaling, enabling models to focus on the most informative synthetic samples. Rather than generating a large dataset and pruning it afterward, DP efficiently approximates the direct generation of informative samples. We theoretically show how training on challenging, informative examples improves scaling laws and empirically validate that DP achieves better scaling performance with significantly fewer training samples and iterations. On ImageNet-100, DP generates 3.4$\times$ fewer samples and requires six times fewer iterations, while on ImageNet-1k, it generates 8$\times$ fewer samples with a 30$\%$ reduction in iterations, all while achieving superior performance compared to prior work.

%% file: 1-intro.tex
\section{Introduction}

A key principle underlying learning in human is deliberate practice (DP)—progress is made not by repeating what is already known but by continuously engaging with tasks that stretch the limits of one’s abilities~\citep{ericsson1993role}. For example, when learning to play the guitar, simply practicing songs that one has mastered does little to improve skill. Instead, targeted practice on challenging tasks and refining learning through feedback, leads to real progress. This principle highlights that effective learning requires exposure to informative and difficult examples rather than passive repetition.

In contrast, most machine learning models are trained on pre-collected data that remain static throughout training, limiting their ability to dynamically adapt to their own weaknesses. One promising source of data for visual recognition tasks is large-scale pre-trained text-to-image models~\citep{rombach2022high}. They provide an essentially infinite source of synthetic training data, presenting an alternative to real-world datasets, which are often expensive or infeasible to curate~\citep{hemmat2023feedback, shin2023fill, zhang2024expanding}. With the great promise of text-to-image models, a natural question arises: what is the potential of learning using \textbf{only} synthetic data? Empirical studies show that increasing the volume of synthetic training data often leads to diminishing returns, with performance gains following a power law stagnation~\citep{fan2024scaling, tian2024learning}. Instead, pruning to remove uninformative examples has proven effective in improving the effectiveness of training with real or synthetic data~\citep{sorscher2022beyond,kolossov2024towards, feng2024modelcollapsescalingsynthesized}. 

Inspired by human learning principles and recent advances in generative image models, we propose the Deliberate Practice (DP) for Synthetic Data Generation framework. Unlike static approaches that generate all synthetic training data upfront~\citep{fan2024scaling, shin2023fill, hemmat2023feedback}, our framework incorporates a dynamic loop between a diffusion model and a downstream learner throughout the training. More concretely, rather than generating an entire dataset at once and irrespective of the learner and then pruning it to remove uninformative samples, we propose DP to efficiently \emph{generate data directly from the pruned distribution of informative samples}. By leveraging the learner's prediction entropy to guide the generation process, our approach generates only the most challenging and informative training examples.

Our framework operates \textbf{dynamically}: we begin with an initial set of synthetic data and train a learner until performance on a real validation set plateaus. At this point, the learner's entropy is used to guide the diffusion model to generate new challenging examples. These examples are added to the training set, and the process repeats, ensuring that the model is continually exposed to increasingly informative data throughout training.

This approach aligns with broader goals in machine learning, such as interactive learning environments, continual learning~\citep{kirkpatrick2017overcoming}, and active learning \citep{settles2009active}. By leveraging a dynamic loop, Deliberate Practice reduces inefficiencies from redundant or already learned data, thereby improving the scaling laws of training with synthetic data.

Our contributions are summarized as:
\begin{itemize}[noitemsep]
    \item We introduce the \textit{Deliberate Practice for Synthetic Data Generation} framework, which dynamically adds new data points when the learner's validation accuracy plateaus [Section~\ref{sec:dp}]. Our framework leverages the learner's prediction entropy to generate\textbf{ challenging synthetic data}, improving the scaling behavior of synthetic data (Figures \ref{fig:diagram} and \ref{fig:scaling-laws}).
    \item We provide a theoretical analysis of the scaling behavior of a simple model trained on selected examples (Section~\ref{sec:3}). Using random matrix theory, we characterize the test error as a function of data size and the example selection function, showing \textbf{improved scaling when prioritizing hard and informative examples}.
    \item We show that entropy-guided sampling approximates generating from an entropy-pruned distribution (Section \ref{sec:2}). We empirically validate that DP can improve the validation accuracy compared to direct pruning while being remarkably \textbf{cheaper in compute up to 5$\times$} (Figure \ref{fig:explicit_prune_vs_DP}).
    \item We demonstrate that DP outperforms prior work on both ImageNet-100 and ImageNet-1k while requiring significantly less data and fewer training iterations. On ImageNet-100, our approach generated \textbf{3.4$\times$ less samples} and completed training in only one-sixth of the iterations used in prior work, yet still achieved superior performance. Similarly, on ImageNet-1k, we generated \textbf{8$\times$ less samples} and reduced the number of iterations by 30\%, while outperforming previous results (Table~\ref{tab:compare}).
    \item Furthermore, DP exhibits strong performance on \textbf{out-of-distribution} (OOD) datasets, even outperforming models trained with real data on ImageNet-R and ImageNet-Sketch, with \textbf{improvements of up to 15\%} (Table~\ref{tab:compare}).
\end{itemize}

\begin{figure}[t!]
    \centering
\includegraphics[width=1.0\linewidth]{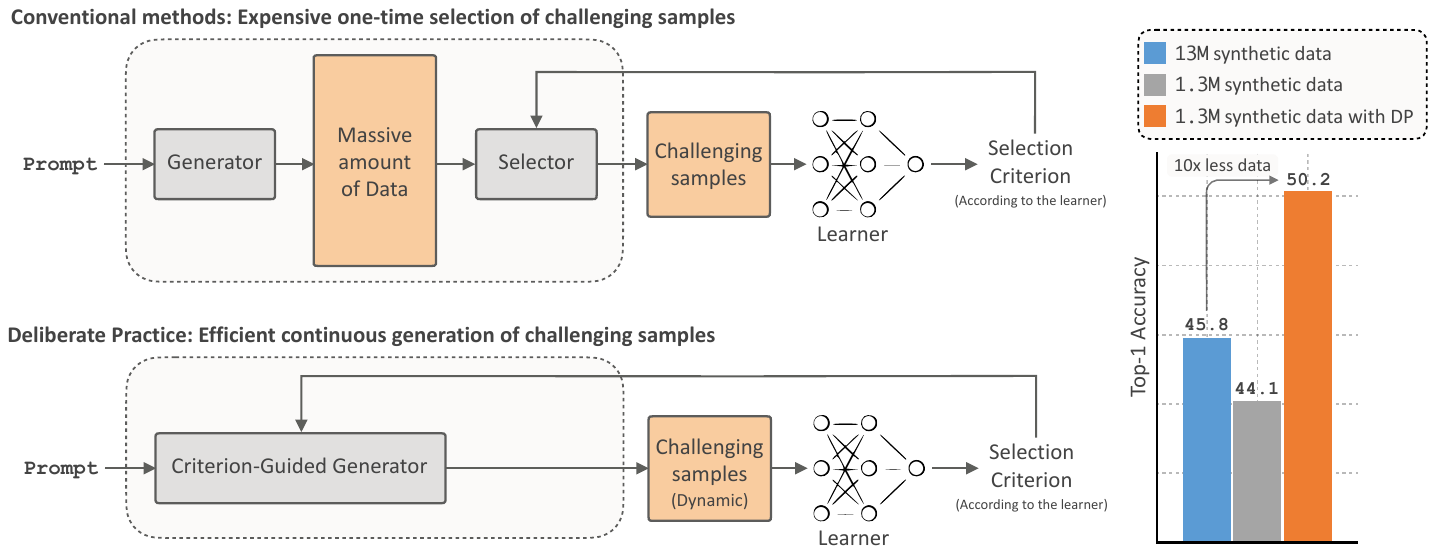}
    \vspace{-0.5cm}
    \caption{
    (\textbf{Top}): Conventional approaches generate (or collect) a massive static dataset and then select challenging examples in a one-time filtering step based on the learner’s selection criterion. This is inefficient, as most generated data is discarded. (\textbf{Bottom}): DP \textbf{continuously generates only the most challenging examples} based on continuous feedback from the learner, eliminating the need for large-scale data pruning. This iterative process ensures that training focuses on progressively informative examples, improving efficiency and performance. (\textbf{Right}): Top-1 validation accuracy on ImageNet-1k with models trained solely on synthetic data. DP (orange) achieves higher accuracy than the 13M synthetic data setup (blue) while using \textbf{10× fewer samples}, significantly outperforming the 1.3M baseline (gray).
    }
    \vspace{-0.2cm}
    \label{fig:diagram}
\end{figure}

%% file: 2-background.tex
\section{Problem Formulation}
\label{sec:2}

\paragraph{Problem Setup.} 
Standard supervised learning relies on a large real labeled training set. Here, however, we assume no real training data is available, and instead, we must rely on a generative model to synthesize training examples.

Formally, let \( \mathcal{Y} \) denote the set of class labels. Our goal is to train a classifier \( f_\phi: \mathcal{X} \to \mathcal{Y} \), parameterized by \( \phi \), which maps inputs \( x \in \mathcal{X} \) (\textit{e.g.}, images) to labels \( y \in \mathcal{Y} \). We are given a predefined label set \( \mathcal{Y} \), a fixed (small) validation set \( \mathcal{D}^{\text{val}} = \{(x_i, y_i)\}_{i=1}^n \) consisting of real data for evaluation, and a generative model \( g_\theta \) capable of sampling synthetic data conditioned on a label, \textit{i.e.}, \( x \sim g_\theta(y) \). However, no real training data is available, \textit{i.e.}, \( \mathcal{D}^{\text{tr}} = \varnothing \). The objective is to train \( f_\phi \) using \emph{as few generated examples as possible} while maximizing generalization to real data as measured by performance on \( \mathcal{D}^{\text{val}} \). The key challenge is to generate minimal yet effective training data, requiring a principled mechanism to select/generate  informative examples.

\paragraph{The Need for Informative Examples.}

Not all synthetic samples contribute equally to learning. Prior work shows that simply increasing the synthetic dataset size leads to diminishing returns, as many generated samples are redundant or too easy ~\citep{fan2024scaling}. Instead, training should focus on examples that maximize learning efficiency.

Given a measure of \textit{informativeness} for a synthetic sample \( x \), one approach is to generate a large dataset and \textbf{prune uninformative examples}. Formally, let \( \mathcal{D}^{\text{pool}} = \{(x_i, y_i)\}_{i=1}^N \) be a large set of $N$ generated samples. We define a \emph{pruned dataset} as $\mathcal{D}' := \{(x_i, y_i) \mid i \in [N], q_i = 1\}$, where $q_i \in \{0,1\}$ is a selection variable determining whether a data point $(x_i, y_i) \in \mathcal{D}^{\text{pool}}$ is retained. The subset size is constrained by \( m = \sum_{i=1}^N q_i \). The quantity $N / m$ is referred to as the over-sampling ratio.

Let $P$ and $Q$ denote the distributions of the original and pruned datasets, respectively. The pruning process operates as an importance sampling scheme:
\begin{equation}
    \label{eq:importance}
    \mathrm{d} Q = \pi \, \mathrm{d} P,
\end{equation}
where $\pi$ is a normalized weighting function that retains the  informative samples. The generate-then-prune approach ensures that only informative examples are kept, it is \textbf{computationally inefficient}, as many generated samples are discarded. This motivates the need to devise mechanisms to directly sample the informative examples.

\paragraph{Approximate  Sampling of Informative Examples.}
Suppose that \( \mathcal{D}^{\text{pool}} \) is generated using a diffusion model with induced probability \( P \). The generative process is governed by a reverse SDE~\citep{song2019generative}:
{
\begin{equation}
\begin{aligned}
    \mathrm{d} x &= \left[v(x, t) - g(t)^2 \nabla \log p_t(x) \right] \mathrm{d} t + g(t) \, \mathrm{d} W(t),
\end{aligned}
\label{eq:standard-reverse}
\end{equation}}
where \( W(t) \) is a Wiener process, modeling stochastic noise, \( v(x, t) \) is a drift term, \( g(t) \) is a coefficient controlling the noise level at time \( t \), and \( \nabla \log p_t(x) \) is the score function.

Instead of sampling from \( P \), we aim to sample directly from \( Q \) as in Eq. \eqref{eq:importance}. By Girsanov’s theorem \citep{oksendal2013stochastic}, modifying the probability measure from \( P \) to \( Q \) introduces a correction term in the reverse SDE:
{
\begin{equation}
\begin{aligned}
    \mathrm{d} x &= \left[v(x, t) - g(t)^2 (\nabla \log p_t(x) + \nabla \log \pi(x, t)) \right] \mathrm{d} t
     + g(t) \, \mathrm{d} W(t).
\end{aligned}
\label{eq:modified-reverse}
\end{equation}}

The term \( \nabla \log \pi(x, t) \) effectively modifies the score function and biases the sampling distribution according to the weighting function \( \pi(x, t) \).  This modification allows approximating direct sampling from the pruned distribution \( Q \), eliminating the need to first sample uniformly from \( P \) and later prune the data.

\subsection{Efficient Entropy-Guided Sampling with DDIM.} 

We leverage denoising diffusion implicit models (DDIMs)~\citep{song2020denoising} for efficient sampling. At each step \( t \), the reverse update for generating a conditional sample is:
{
\begin{equation*}
    x_{t-1} = \sqrt{\xi_{t-1}} \hat{x}_{0, t} + \underbrace{\sqrt{1 - \xi_{t-1} - \sigma_t^2} \cdot \epsilon_\theta^{(t)}(x_t, y)}_{\text{direction pointing to } x_t} + \underbrace{\sigma_t \epsilon_t}_{\text{random noise}},
\end{equation*}}
where \( \epsilon_t \) is random noise and \( \sigma_t \) and \( \xi_{t-1} \) are time-dependent coefficients. The term \( \hat{x}_{0, t} \) approximates the final denoised sample:
\begin{equation}
    \hat{x}_{0, t} = \frac{x_t - \sqrt{1 - \xi_t} \epsilon_\theta^{(t)}(x_t, y)}{\sqrt{\xi_t}},
\end{equation}
in which $\epsilon_\theta^{(t)}(x_t, y)$ approximates the conditional score function using a pretrained denoising network~\citep{ho2022classifier}:
{
\begin{equation}
    \epsilon_\theta(x_t, y) \approx (1+\lambda)\tilde{\epsilon}_\theta(x, y) - \lambda \tilde{\epsilon}_{\theta}(x)
\end{equation}
}
where $\lambda$ is called the classifier-free guidance coefficient which controls the strength of conditional sampling on the label.\looseness-1

An efficient way of sampling from a modified diffusion mode as described in Eq.~\ref{eq:modified-reverse} was proposed by~\citet{hemmat2023feedback}, where the weighting function is derived from the entropy of the downstream learner, such that,
{
\begin{equation}
    \label{eq:ent}
    \log \pi \propto H(f_\phi(x_0)) = - \sum_{y \in \mathcal{Y}} f_\phi(y \mid x_0) \log f_\phi(y \mid x_0).
\end{equation}}
To compute the entropy as in Eq. \ref{eq:ent}, we need the denoised sample $x_0$. The term \( \hat{x}_{0, t} \) can be used to cheaply approximate entropy mid-generation. This allows direct sampling of high-entropy examples by modifying the score function:
{
\begin{equation}
    \tilde{\epsilon}_\theta^{(t)}(x_t, y) = \epsilon_\theta^{(t)}(x_t, y) + \omega \nabla_{x_t} H(f_\phi(\hat{x}_{0, t})),
\end{equation}}
where $\omega$ controls the contribution of the entropy-guidance.

In \cite{hemmat2023feedback}, real data is used to pre-train the learner, enabling an accurate estimation of \( \nabla_{x_t} H(f_\phi(\hat{x}_{0, t})) \). However, when real data is unavailable, alternative approaches are needed to assess sample informativeness. In the next section, we propose to leverage the learner itself during training to evaluate entropy and determine the informativeness of generated samples dynamically.

%% file: 3-method.tex
\begin{figure}[ht]
    \centering
\includegraphics[width=1.0\linewidth]{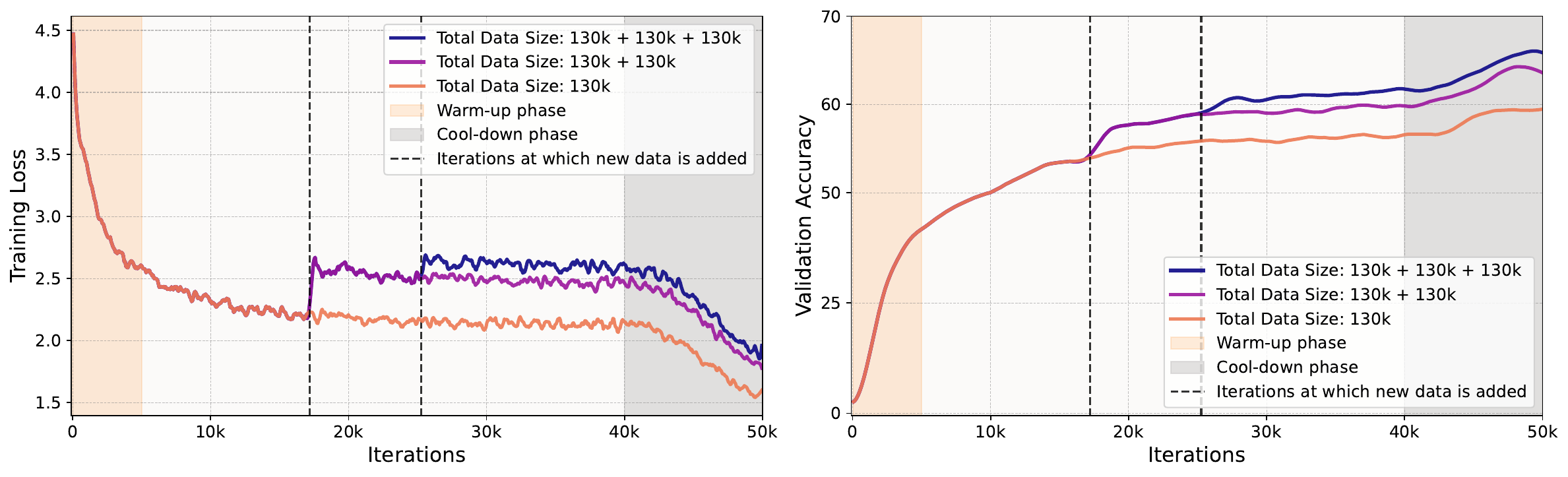}
\caption{
    Training loss (left) and validation accuracy (right) of Deliberate Practice on ImageNet-100.
    The classifier begins training on an initial static dataset (130k samples) until validation accuracy plateaus. At this point, additional samples are generated using entropy-guided sampling, focusing on hard/informative examples. The two dashed vertical lines indicate points where new data is added.
    We compare three setups: 
    (1) \textcolor[rgb]{0.8, 0.47, 0.39}{Orange}: No additional data is added, training only on the initial dataset. 
    (2) \textcolor[rgb]{0.6, 0.23, 0.64}{Purple}: One round of entropy-guided data generation adds 130k samples. 
    (3) \textcolor[rgb]{0.08, 0.04, 0.62}{Blue}: Two rounds of entropy-guided data generation, adding 260k samples in total. 
    Each data addition leads to an accuracy boost, demonstrating the effectiveness of DP in improving performance with fewer training iterations.
    For clarity, this figure shows only two rounds of data addition, but in practice, more rounds occur based on the allowed maximum patience. 
    Notably, while training loss increases with new data, validation accuracy steadily improves, showing that the model benefits from progressively challenging examples, ultimately reducing the generalization gap.
}
    \label{fig:branch_outs}
\end{figure}
\section{The deliberate Practice Framework for Synthetic Data Generation}\label{sec:dp}

In this section, we describe our Deliberate Practice framework, in which we efficiently train the learner with synthetic data in absence of any real data. In particular, we move to a  setup where we dynamically expand the dataset throughout the training.
Our framework is summarized in Algorithm~\ref{alg:framework}.

\begin{algorithm}
  \caption{Deliberate Practice for Synthetic Data Generation}
  \label{alg:framework}
  \textbf{Input:} Class labels $\mathcal{Y}$, Generative model $g_\theta$, Validation set $\mathcal{D}^{\text{val}}$, Initial dataset size $N$, New data size $P$, Patience $T_{\text{max}}$, Evaluation interval $\tau$. \\
  \textbf{Output:} Trained classifier $f_\phi$
  \begin{algorithmic}[1]
    \State \textbf{Initialize:} Generate $\mathcal{D}_0^{\text{tr}}$ with $N$ examples from $g_\theta$. Start training $f_\phi$ with learning-rate warm-up.
    \State Set patience counter $T \gets 0$.
    \While{training}
        \State Update $f_\phi$ on a mini-batch drawn uniformly from $\mathcal{D}_k^{\text{tr}}$.
        \If{(every $\tau$ iterations)}
            \State Evaluate validation accuracy $\mathcal{A}(f_\phi, \mathcal{D}^{\text{val}})$.
            \State Reset $T \leftarrow 0$ if accuracy improves; else increment $T \leftarrow T + 1$.
        \EndIf
        \If{$T \geq T_{\text{max}}$}
            \State Generate $P$ new examples $\mathcal{D}_{\text{new}}$ with feedback:
            \[
            \nabla_{z_t} \log p(x_t | y) = \nabla_{z_t} \log p_\theta(z_t) + \omega \nabla_{z_t} H(f_\phi(\hat{x}_{0, t}))
            \]
            \State Augment training set: $\mathcal{D}_{k+1}^{\text{tr}} \gets \mathcal{D}_k^{\text{tr}} \cup \mathcal{D}_{\text{new}}$.
            \State Reset $T \gets 0$.
        \EndIf
    \EndWhile
    \State \textbf{Finalize:} Apply learning rate decay.
  \end{algorithmic}
\end{algorithm}

\textbf{The initial training data.}
The framework begins by generating an initial set of $N$ synthetic training examples $\mathcal{D}^{{tr}}_0 = \{(x_i, y_i)\}_{i=1}^N$ using a pre-trained generative model $g_\theta$. For each class $y_i \in \mathcal{Y}$, the generative model samples  images $x_i \sim g_\theta(y_i)$ in a class-conditional manner. The classifier $f_\phi$ starts training on this dataset, with a learning-rate warm-up phase.

\textbf{Iterative training and additional data.}
Training proceeds iteratively with a mechanism to dynamically augment the dataset whenever the classifier's performance stagnates. The process alternates between training the classifier and generating new synthetic examples.

\textbf{Patience mechanism.} 
At regular iteration intervals, $\tau$, the validation accuracy $\mathcal{A}(f_\phi, \mathcal{D}^{\text{val}})$ is evaluated. If no improvement is observed for $T_{\text{max}}$ intervals (patience threshold), the framework triggers new data generation. 

\textbf{Entropy guided sampling.}  
When the patience mechanism triggers, $P$ new examples $\mathcal{D}_{\text{new}} = \{(x_j, y_j)\}_{j=1}^P$ are generated. We directly generate samples from the entropy pruned distribution through entropy guided sampling. The entropy is computed based on the current stage of the classifier $f_\phi$. The $\omega$ coefficient controls the effect of entropy-guidance. With $\omega=0$, we fall back into regular sampling of diffusion models, while  $\omega>0$ results in generations that have a higher entropy under the classifier.

\textbf{Training resumption.}  
The newly generated examples are added to the dataset, $\mathcal{D}_{k+1}^{\text{tr}} = \mathcal{D}_k^{\text{tr}} \cup \mathcal{D}_{\text{new}}$. After augmenting the dataset, training resumes with a constant learning rate until the patience mechanism is triggered again. Mini-batches are drawn uniformly from the updated pool, which grows dynamically from size $N$ to $N + kP$ after $k$ iterations of augmentation. This cycle is continued until we reach the cool-down phase where the learning rate is decreased and no more new data is added. See Figure~\ref{fig:branch_outs} for training dynamics of a classifier training with DP.

In Section \ref{sec:3}, we provide an intuitive theoretical framework to study the scaling behavior of a simplified DP. In Section \ref{sec:exp}, we validate the effectiveness of DP in large-scale experiments.

%% file: 5-theory.tex
\section{Training on informative examples improves the scaling laws}\label{sec:3}

Before presenting empirical results, we first analyze how selecting informative examples affects the scaling of synthetic data. We study a high-dimensional linear classifier trained with uniform vs. selective sampling and derive an analytic expression for test error using random matrix theory (RMT). Our results show that selecting hard examples improves scaling laws, providing theoretical justification for our approach.

\subsection{Theoretical Analysis under an Idealized Setup.}

Consider a simple generative model for training data:
\begin{equation}
    x \sim \mathcal{N}(0, \Sigma), \quad y = \text{sign}(w_0^\top x),
\end{equation}
where $w_0 \in \mathbb{R}^d$ is the ground-truth labeling function. This gives a distribution $P$ on $\mathbb R^{d} \times \mathbb R$.

We study the impact of \textit{uniform sampling} versus \textit{selective sampling} of informative examples on generalization. To formalize this, we assume a pool of $n$ i.i.d. training pairs:
{
\begin{equation}
    X \in \mathbb{R}^{n \times d}, \quad Y \in \mathbb{R}^{n}.
\end{equation}}
A linear classifier $\hat{w}$ is trained using the following loss:
\begin{equation}
    \hat{w} = \underset{w}{\arg\min} \quad \frac{1}{n} \sum_{i=1}^{n} q_i \ell(w^\top x_i, y_i) + \frac{\lambda}{2} \|w\|^2.
\end{equation}
where $\ell(z, y) = (z - y)^2/2$ is the squared loss, $\lambda > 0$ is a regularization parameter, and $q_i := q(x_i^\top w_s)$ is a selection strategy that determines whether an example is included in training based on its projection in a given direction $w_s \in \mathbb R^d$, and an arbitrary measurable binary function $q:\mathbb R \to \{0,1\}$ which encodes the selection strategy.

The \textit{selection/pruning ratio} is given by:
{
\begin{equation}
    p = \mathbb{E}[q(x^\top w_s)]\text{ for }x \sim \mathcal{N}(0, \Sigma).
\end{equation}}
The resulting classifier has a closed-form solution:
\begin{equation}
    \hat{w} = \frac{1}{n} R X^\top D Y, \quad R := \left(\frac{1}{n} X^\top D X + \lambda I_d \right)^{-1},
    \label{eq:estimator}
\end{equation}
where $D \in \mathbb{R}^{n \times n}$ is a diagonal matrix with $D_{ii} = q_i$.

Our objective is to analyze the asymptotic test error of $\hat{w}$:
\begin{equation}
    E_{test}(\hat{w}) = \mathbb{P}(\text{sign}(x^\top \hat{w}) \ne y),
    \label{eq:Etest}
\end{equation}
where $(x, y)$ is a test example,

\subsection{Asymptotic Behavior of the Test Error.}
We leverage random matrix theory (RMT) techniques \citep{Couillet_Liao_2022, ZhenyuAndMahoney2021, Firdoussi2024} to characterize the test error in Eq. \eqref{eq:Etest}. Our analysis is based on the spectral density of the resolvent matrix $R$ in Eq. \eqref{eq:estimator}, allowing us to compute the first two moments of $yx^\top \hat{w}$ for a test sample $x$ and derive an expression for the test error. For simplicity, we assume an isotropic setup where $\Sigma = I_d$ and defer the general case to Appendix~\ref{app:theory}.

We shall work in the following so-called high-dimensional proportionate scaling regime
\begin{equation}
\label{eq:proportionate}
    d,n \to \infty,\quad d/n \to \phi,
\end{equation}
in which the input-dimension $d$ and the sample size $n$ diverge to infinity at the same rate.
The scalar $\phi \in (0,\infty)$ captures the effective dimensionality or over-parametrization rate of the problem.

\textbf{Key Scalars.}
WLOG, assume $\|w_s\|=1$. It turns out that the for fixed, pruning, $p$, the asymptotic test error is fully captured by the following scalars:
{
\begin{equation}
    \begin{split}
    \rho &:= w_s^\top w_0/\|w_0\|, \, \tau := \frac{\rho}{\sqrt{1-\rho^2}},\, \gamma := \mathbb{E}[q(G)G^2],\\
    \beta &:= 2\mathbb{E}[q(G)\varphi(\tau G)],\quad  \tilde \beta := 2\mathbb{E}[q(G)\Phi(\tau G)G],
\end{split}
\end{equation}
}
where $G \sim \mathcal N(0,1)$ with pdf $\varphi$ and cdf $\Phi$. Note that $\rho$ quantifies the alignment between the pruning direction $w_s$ and the ground-truth labeler $w_0$, while $\beta$ and $\gamma$ capture statistical properties of the pruning strategy $q$.

\textbf{Spectral functions.} The Stieltjes transform $m$ of the limiting spectral density of the resolvent matrix $R$ is shown in Lemma \ref{lm:mp} to be given by the exact formula (with $z:=-\lambda$)
\begin{equation}    
 m(z) = \frac{p-\phi-z - \sqrt{(p-\phi-z)^2-4\phi z}}{2\phi z},
\label{eq:meq-mp}
\end{equation}
and will play an important role in our theory.
 The above formula represents a somewhat distorted Marchenko-Pastur law. 
 Indeed, the classical MP \citep{MP1967} corresponds to $p \to 1$ (i.e. no data pruning).
 
We further define the following auxiliary functions:
{
\begin{align}
    s(z) &:= \gamma/(1+\phi m(z)),\quad
    \tilde m(z) := 1/(s(z)-z), \\
    r(z) &:= \omega^2\cdot m(z) + \tilde\omega ^2\cdot \tilde m(z),\\
    \text{with }\omega&:=\sqrt{1-\rho^2}\beta,\quad \tilde\omega := \rho\tilde\beta.
\end{align}}

\paragraph{Main Result: Test Error Scaling w.r.t Selection Strategy.}
\begin{theorem}
\label{thm:main}
   In the limit Eq. \eqref{eq:proportionate}, the classification test error satisfies: $E_{test}(\hat{w}) \to \Phi\left(-m_0/\sqrt{\nu_0 - m_0^2}\right)$, where
    \begin{align*}
        m_0 &:= \sqrt{2/\pi} \cdot r(-\lambda),\\
        \nu_0 &:= p\phi m'(-\lambda) + r'(-\lambda) - \frac{2\phi m'(-\lambda)}{1+\phi m(-\lambda)} r(-\lambda).
    \end{align*}
\end{theorem}
The scaling behavior of test error is fully determined by the six scalars $(\lambda, \phi, p, \rho, \gamma, \beta,\tilde\beta)$. Importantly, the choice of the data point selection strategy $i \mapsto q(x_i^\top w_s)$ only influences performance through $\rho$, $\gamma$, $\beta$, and $\tilde \beta$.

\begin{figure}[t!]
    \centering
    \begin{minipage}{0.4\textwidth}
        \includegraphics[width=\linewidth]{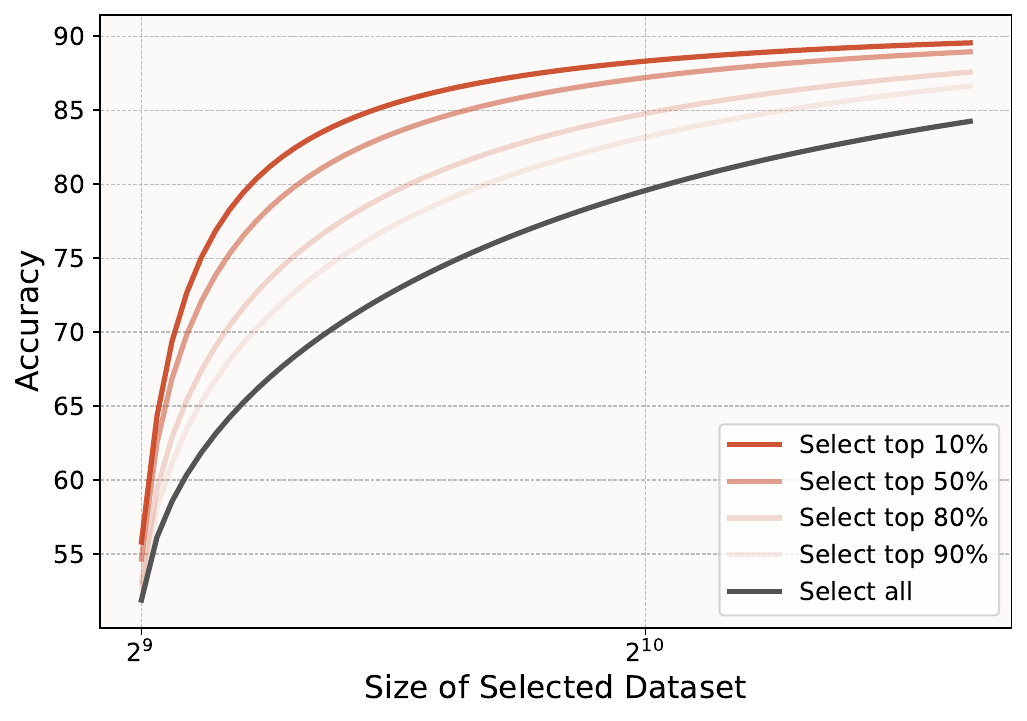}
    \end{minipage}
    \hfill
    \begin{minipage}{0.55\textwidth}
    \caption{Theoretical prediction for scaling behavior of accuracy (Theorem~\ref{thm:main}) for a simple classifier in a $d=512$ dimensional input space, as a function of dataset selection strategy. The classifier is trained on synthetic data with different pruning probabilities, where higher pruning probability corresponds to keeping only the most challenging examples (those closer to the decision boundary). The results compare selecting all samples (gray) versus selecting a fraction of the hardest samples (red). Selecting harder examples improves sample efficiency, achieving higher accuracy with fewer training samples.}
        \label{fig:theory_scaling}
    \end{minipage}
\end{figure}

\subsubsection{Example: Selecting Informative Examples.}

Consider a selection function of the form $q_i = q(x_i^\top w_s)$ for all $i$, where,
\begin{eqnarray}
q(t) := 1[|t| \le \xi] =  \begin{cases}
    1,&\mbox{ if }|t| \le \xi,\\
    0,&\mbox{ else,}
    \end{cases}
\end{eqnarray}
for some threshold $\xi \ge 0$. Such selection strategy selects only the examples near the decision boundary of $w_s$, analogous to using classifier entropy as a selection criterion but simpler to study. Lemma ~\ref{lm:KH} and \ref{lm:KE} derive explicit expressions for $(\gamma,\beta,\tilde\beta)$. 
Figure~\ref{fig:theory_scaling} presents theoretical predictions for test accuracy across different degrees of example selection, showing that \textit{selecting hard examples improves scaling laws}, reducing the number of training samples needed for the same performance. However, beyond a certain point, excessive pruning degrades performance, as illustrated in Figure~\ref{fig:explicit_prune_vs_DP}.

\subsubsection{Adaptive Selection Strategy.}  
Data selection relies on a pruning direction \( w_s \) to select informative/hard examples: $i \mapsto q(x_i^\top w_s) \in \{0,1\}$, but these examples are ultimately used to train \( \hat{w} \). If \( w_s \) and \( \hat{w} \) are misaligned, what is considered hard by \( w_s \) may not be hard for \( \hat{w} \), reducing the effectiveness of selective sampling. In fact, hard examples change over time: an example that was identified hard, might not remain hard are more training is done. To ensure alignment, \( w_s \) should periodically update to reflect the evolving decision boundary of \( \hat{w} \). This adaptive selection mechanism motivates the continuous data generation process of DP, as presented in Section \ref{sec:dp}.

Data selection relies on a pruning direction \( w_s \) to identify informative or hard examples: \( i \mapsto q(x_i^\top w_s) \in \{0,1\} \). However, these selected examples are ultimately used to train \( \hat{w} \), and if \( w_s \) and \( \hat{w} \) are misaligned, what is considered hard by \( w_s \) may not be hard for \( \hat{w} \), reducing the effectiveness of selective sampling. In fact, \( w_s \) and \( \hat{w} \) deviate from each other the more \( \hat{w} \) is trained on these examples. Moreover, the definition of ``hard'' changes over time—an example that was initially difficult may become easier as training progresses. To maintain alignment, \( w_s \) should be periodically updated to reflect the evolving decision boundary of \( \hat{w} \). This adaptive selection mechanism underpins the continuous data generation process in DP, as presented in Section~\ref{sec:dp}.

%% file: 4-experiments.tex
\section{Experiments}\label{sec:exp}
For all the experiments, we use the LDM1.5 ~\citep{rombach2022high} as the
pre-trained text-to-image (T2I) model. We studied four  different T2I models and found this model outperforming the rest. For more details see Appendix~\ref{app:choice_of_gens}.

\textbf{Datasets.}
We validate our framework on two datasets. ImageNet-100~\citep{tian2020contrastive, sariyildiz2023fake}, a subset of ImageNet-1k~\citep{deng2009imagenet}, containing 100 classes and 5,000 validation examples, where the real validation set is used for evaluation and the real training set (126,689 examples) serves as a held-out test set.
We also conduct experiment ImageNet-1k, using the 50,000 validation examples to monitor performance and reserving the real training set (1.3 million examples) as a held-out test set.

\begin{figure*}
\centering
    \includegraphics[width=.49\linewidth]{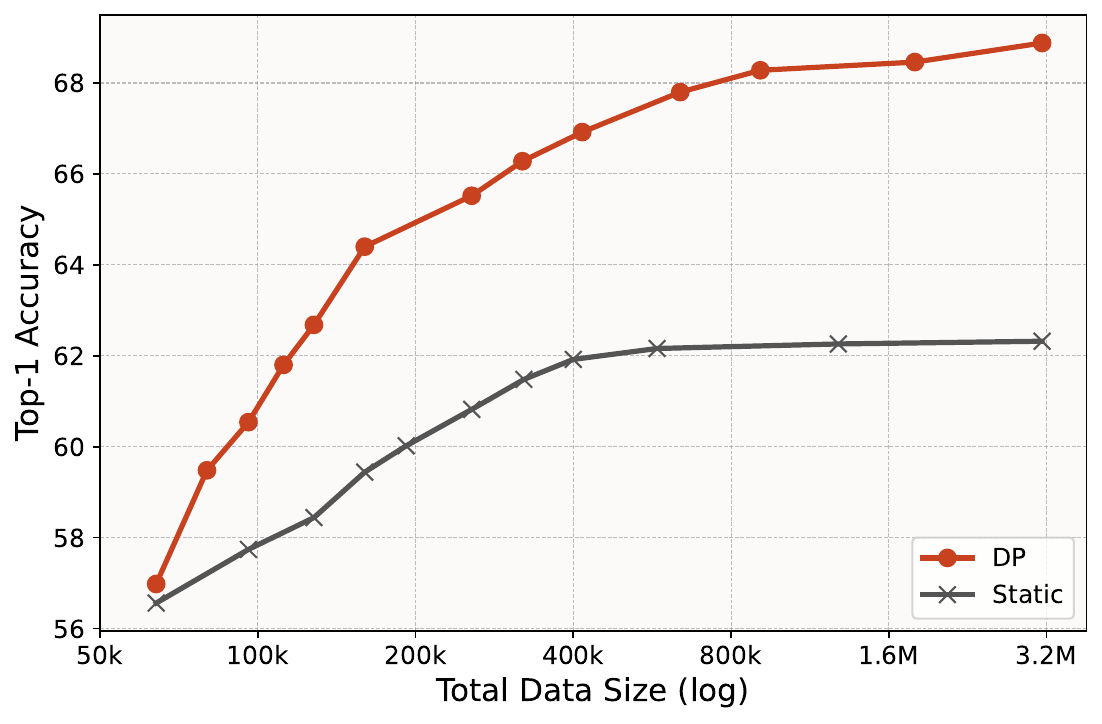}
    \hspace{0mm}
        \includegraphics[width=.49\linewidth]{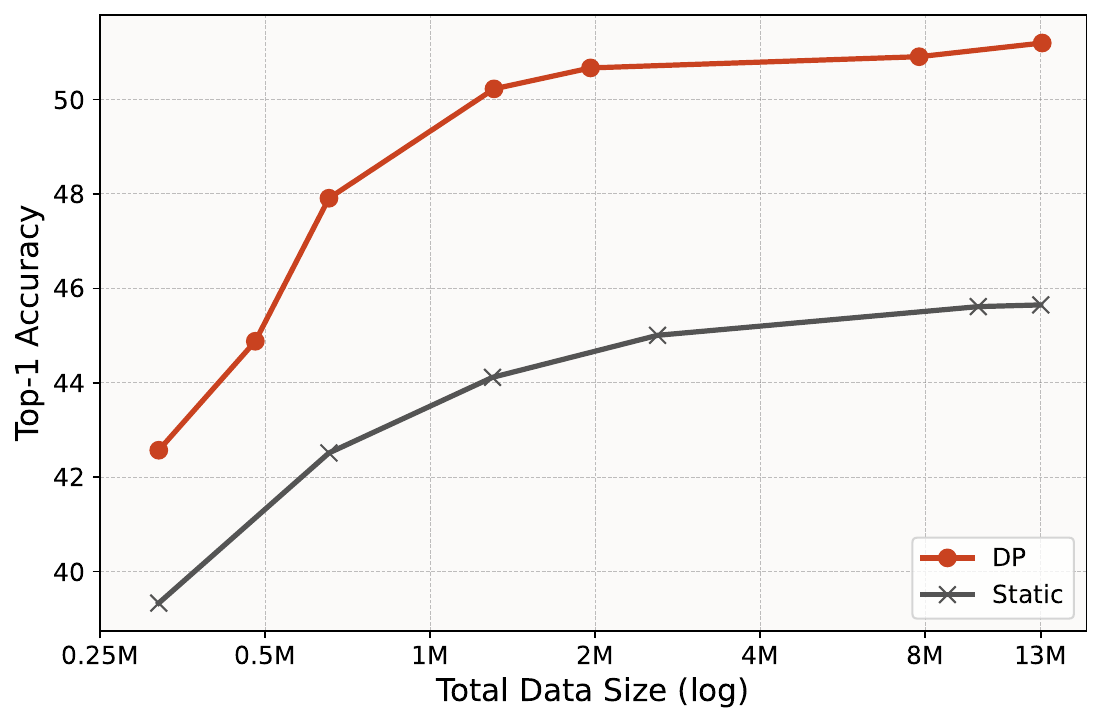}
    \caption{\textbf{Scaling laws of synthetic data.} 
        Real Validation accuracy versus total dataset size  for the Static (pink $\times$), and Deliberate Practice (blue o) setups on ImageNet-100 (left) and ImageNet-1k (right). DP significantly outperforms Static data generation, achieving higher  accuracy with fewer synthetic examples. DP achieves the same accuracy as the static setup using 7.5$\times$ less data on ImageNet-100 and 20$\times$ less data while outperforming it on ImageNet-1K.
        }
\label{fig:scaling-laws}
\end{figure*}

\subsection{Scaling Laws of Synthetic Data}
We train a Vision Transformer (ViT-B)~\citep{dosovitskiy2020image} classifier with synthetic data. We study two scenarios: 1) Static data generation and 2) Deliberate Practice (DP). In all the experiments in this section we have a fixed and controlled setup. We train the models for 100k and 50k iterations for ImageNet-1k and ImageNet-100 respectively. For additional details, see Appendix~\ref{app:exp_details}.

\textbf{Static data generation.} 
In this setup, all data is generated before training, and the classifier is trained on a fixed dataset. We experiment with different dataset sizes to see its impact on accuracy.

\textbf{Deliberate Practice data generation.}
Hyperparameters \(\omega\) and \(\lambda\) are tuned on ImageNet-100 and found effective for ImageNet-1k as well (see Section \ref{app:exp_details} for details). We track validation accuracy throughout training and use it to determine when to generate new data, following a patience-based criterion. To ensure the model has not over-fitted to the validation set, we also report accuracy on the full real training sets of ImageNet-100 and ImageNet-1k, used as \textit{held-out} test sets.

Figure~\ref{fig:scaling-laws} compares the scaling laws of the  \textbf{Static} and \textbf{Deliberate Practice (DP)} on ImageNet-100 and ImageNet-1k. On both datasets, we note that DP scales well with dataset size and it consistently outperforms the Static setup, achieving higher validation accuracy at any given dataset size. On ImageNet-100 we observe that DP can reach the best accuracy of the static setup (with 3 million examples) using only 400k examples. This means that DP requires 7.5$\times$ less data to reach the same performance. On ImageNet-1k, we observe that DP can outperform the best accuracy of the static setup (with 13 million examples), using only 640k examples. This translates to DP requiring 20$\times$ less data to outperform the Static setup. For additional details on the hyper-parameters of these experiments, see Appendix~\ref{app:details_scaling}. Refer to Figure~\ref{fig:vis_init_vs_final} for a visualization of how the dataset evolves from the start to the end of training.

\begin{table*}[h]
\centering
\caption{\textbf{Comparison with previous work.} DP outperforms other models on both ImageNet-100 and ImageNet-1k while requiring significantly less data and fewer training iterations. Note that DP experiments reported in this table are trained longer than models reported in the previous section and, consistent with other work, use a smaller classifier free guidance scale of $\lambda=2$.}
\vspace{0.1cm}
{\scriptsize
\begin{tabular}{llrrllllll}
\toprule
 & Task & \# Iters & Data size & IN real Val.& IN real tr.& IN-v2 & IN-Sk & IN-R & IN-A \\
\midrule
\midrule
Real & IN-100 & 100k & 130k & 88.5 & - & \textbf{76.4}&  37.1&  60.8& \textbf{33.5} \\
Syn. Static - \cite{sariyildiz2023fake}& IN-100 & 13k & 130k & 63.5 & - & 62.7&  41.8&  64.2& 13.7 \\
Syn. Static - \cite{sariyildiz2023fake} & IN-100 & 635k & 6.5M & 73.3 &- &72.3& 42.0& 59.4 & 17.1  \\
Syn. DP (ours) & IN-100  &100k & 1.9M & \textbf{74.3} & 75.0 & 66.3 &  \textbf{52.0}&  \textbf{76.6}& 25.9 \\
		
\midrule
Real & IN-1k & 200k & 1.3M & 82.6&-& \textbf{70.9}&  32.5&  44.6& \textbf{29.4} \\
Syn. Static - \cite{sariyildiz2023fake}& IN-1k  & 130k& 1.3M& 42.9& - & 43.0 &  16.6 &  26.3& 3.6 \\
Syn. Static - \cite{fan2024scaling}& IN-1k  & 210k& 2M & 50 & - & 42.2 & 27.2 & 45.7 & 6.6\\
Syn. Static - \cite{fan2024scaling}& IN-1k  & 315k& 64M& 54 & - & 46.0 & 32.4 &52.5 & 9.4\\
Syn. DP (ours)& IN-1k  & 200k & 6.5M & {54.1} & 54.84 &48.5 & 34.7 &  56.0&  12.3 \\
Syn. DP (ours) & IN-1k  & 200k & 9.1M & \textbf{55.1} & 55.73 &{49.3}&\textbf{36.0} &\textbf{57.2} & {13.4} \\
\bottomrule
\end{tabular}\label{tab:compare}
}
\end{table*}

\subsection{Comparison with Previous Work}
We compare DP with prior works on synthetic data generation for image classification~\citep{sariyildiz2023fake, fan2024scaling}. Specifically, we evaluate setups that use classnames for prompting and publicly available models for sample generation. Performance is assessed on real ImageNet (held-out) training and validation sets, as well as on ImageNet-V2~\citep{recht2019imagenet}, ImageNet-Sketch~\citep{wang2019learning}, ImageNet-R~\citep{hendrycks2021many}, and ImageNet-A~\citep{hendrycks2021natural} to measure out-of-distribution (OOD) generalization.

The results in Table~\ref{tab:compare} show that DP outperforms prior benchmarks on both ImageNet-100 and ImageNet-1k while requiring significantly less data and fewer training iterations. On ImageNet-100, DP generated 4.6 million fewer samples and trained for only one-sixth of the iterations compared to previous works, yet achieved superior performance on the real data. Similarly, on ImageNet-1k, DP reduced sample generation by 56.2 million and cut training iterations by over 30\%, while still outperforming previous results.

Furthermore, models trained with DP exhibit strong performance on out-of-distribution datasets, even surpassing models trained on real data on ImageNet-R and ImageNet-Sketch, with improvements of up to 15\%.

\subsection{Connection Between Pruning and DP}

In Section~\ref{sec:2}, we discussed how DP approximates direct sampling from a pruned distribution. Here, we validate this experimentally on ImageNet-100 using two setups:
\begin{enumerate}
    \item \textbf{Oversampling then Pruning:} Generate a large pool and select high-entropy samples.
    \item \textbf{Direct entropy-guided generation:} Generate only informative samples (a special case of DP with a single step of data addition).
\end{enumerate}

We start with 130k generated samples (regular vanilla sampling), train for 17k iterations, then add a one-time additional 130k samples, increasing the total data size to 260k and training for an additional 33k iterations.

In setup 1, we vary the pool size, ranging from no pruning (130k pool) up to an oversampling ratio of 18 (2.4M pool), selecting the top 130k high-entropy samples. In setup 2, we generate exactly 130k entropy-guided samples, varying the entropy-gauidance coefficient.

Figure~\ref{fig:explicit_prune_vs_DP} (a, b) shows that both methods improve performance up to a point, after which excessive selection of high-entropy samples leads to degradation—likely due to selecting high-entropy but harmful outliers. This aligns with our theoretical predictions in Figure~\ref{fig:explicit_prune_vs_DP} (c).

Regarding computational costs, generating a single image with entropy-guidance on an Nvidia H100 takes 1.82$\times$ longer than standard vanilla sampling. However, achieving similar performance through oversampling requires significantly more data, leading to a linear increase in cost. As a result, DP is $5\times$ more efficient while also providing higher absolute improvements compared to pruning-based selection. See Figure~\ref{fig:explicit_prune_vs_DP} for details and Figure~\ref{fig:wolfs} for some visualizations.

\begin{figure}[th]
    \centering
    \includegraphics[width=1\linewidth]{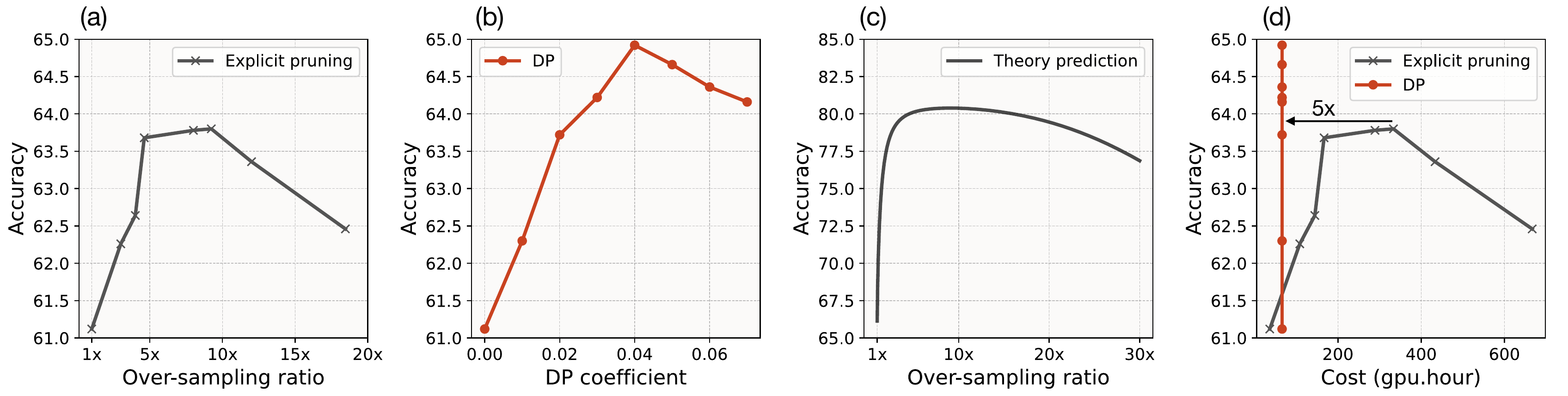}
    \caption{Plots describing the performance of DP compared to explicit pruning and theory prediction while changing the oversampling ratio or the DP coefficient. (a) Over-sampling with entropy-based selection – Generate a large pool of samples (ranging from 130k to 2.4M) and select the 130k highest-entropy examples. (b) Generate 130k high-entropy examples directly using DP with varying entropy guidance strength through $\omega$. (c) The theory prediction on the accuracy based on the over-sampling ration. (d) Comparing the compute cost of DP vs oversampling then pruning.
    We observe that DP exhibits a similar accuracy curve compared to explicit pruning and theoretical prediction when changing the over-sampling/DP coefficient. However, DP is computationally remarkably more efficient while gaining more accuracy delta.}
    \label{fig:explicit_prune_vs_DP}
    \vspace{-0.2cm}
\end{figure}

\subsection{The evolution of hard examples over time}

``Does the sample hardness change as training progresses?''

To answer this question, Figure~\ref{fig:everystep} (left) tracks the error on examples that were misclassified at the time they were added. As expected, once introduced, the model gradually learns to classify them correctly. However, an interesting trend emerges: even before these examples were added, their error was lower than at the moment of inclusion. This suggests that the notion of hardness is dynamic—what is considered challenging at one point may become easier over time. Conversely, examples that were once easy might later become difficult due to shifts in the learned decision boundaries. This highlights a key limitation of static pruning approaches and underscores the importance of dynamically adapting the selection of informative examples throughout training, as done in Deliberate Practice (DP). See Figure~\ref{fig:vis_stages} for some visualization of generations through training.

Figure~\ref{fig:everystep} (right) shows the evolution of generated examples from the class of ``school bus'' throughout training. Early samples often have atypical colors or grayscale tones, indicating the model's initial struggle with changes in the \textit{color} features. As training progresses, more challenging examples with unusual shapes and viewpoints emerge, reflecting the model’s shifting focus towards more complicated features such as \textit{shape}. See additional samples in Figure~\ref{fig:merged}.

\begin{figure}[ht]
    \centering
\includegraphics[width=1\linewidth]{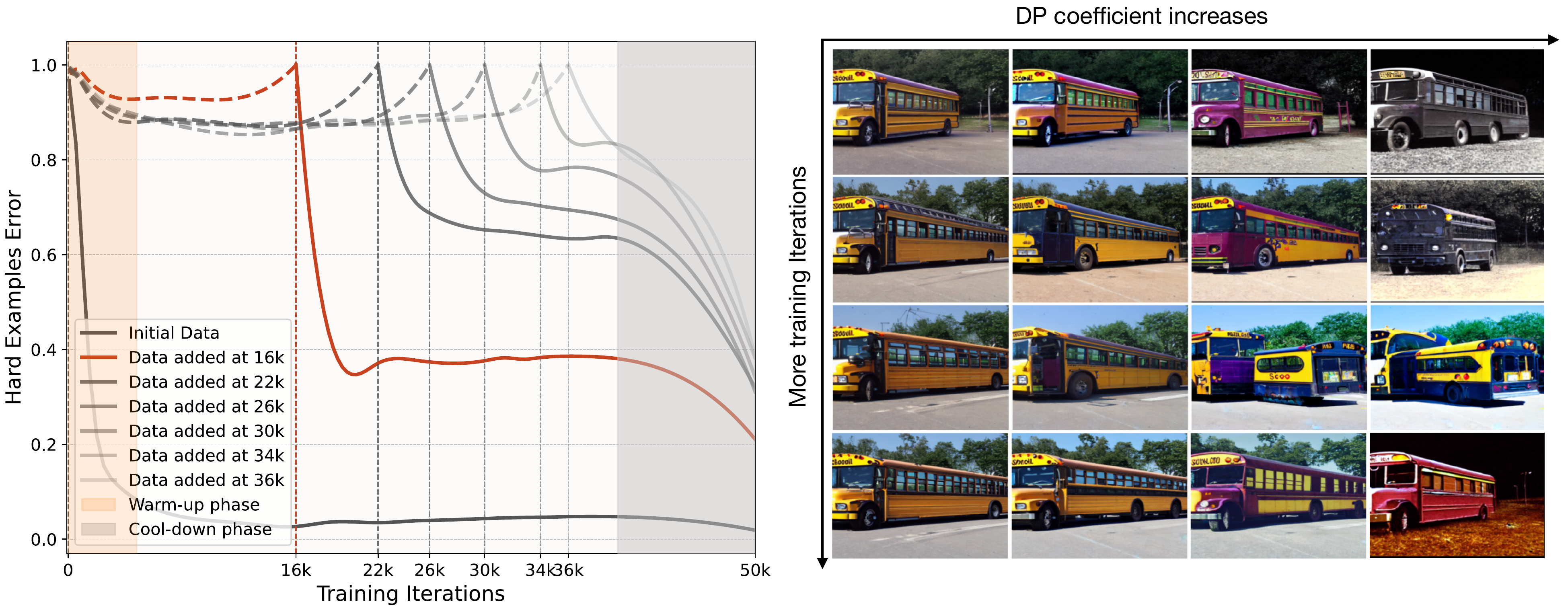}
    \caption{\textbf{Left:} Error trajectories of hard (misclassified) examples added at different training stages. The red curve highlights the first batch of added data for better visibility, but the same trend applies to all batches. Notably, even before being trained on, these examples exhibit a lower error rate than at their point of inclusion, indicating that hardness is not static, it evolves throughout training. \textbf{Right:} Examples of the prompt ``school bus'' generated at different epochs with varying entropy guidance scales (\(\omega\)). We observe that the model's training needs evolve over time, starting with simpler challenges like color recognition before progressing to harder examples with unusual shapes or viewpoints. }
    \label{fig:everystep}
\end{figure}

%% file: 6-related_work.tex
\section{Related Work}
\textbf{Synthetic data for training neural networks.} Synthetic data has become a powerful tool for training machine learning models across various domains. For instance, text-to-image diffusion models have been successfully used for visual representation learning~\citep{astolfi2023instance, li2025genview, tian2024learning, tian2024stablerep, sariyildiz2023fake}. However,  limitations of synthetic data are highlighted by~\citet{fan2024scaling}, emphasizing the importance of generating more challenging and informative examples. Addressing distribution shifts between synthetic and real data, \citet{hemmat2023feedback} and \citet{yuan2023real} propose synthesizing training data that matches real data distributions or conditioning on real examples to reduce this gap. Expanding small-scale datasets has also been studied, see e.g.\ ~\citet{zhang2024expanding}.
Another related line of work involves using VLMs and LLMs to generate descriptions for augmenting datasets~\citep{dunlap2023diversify}.

Synthetic data is increasingly used to train (LLMs). For example, LLaMA3~\citep{grattafiori2024llama3herdmodels} employs AI-generated data for fine-tuning. Similarly, self-play approaches, e.g.,\ \citet{yuan2024self}, align with our framework by generating increasingly difficult examples for training.

\textbf{Continual learning and active learning.}
 Our work is also closely related to principles from active learning~\citep{bang2024active,evans2023bad} and continual learning, which prioritize iterative model updates with tailored data. These methods highlight the importance of selecting informative samples based on the model's current state.
 \cite{sorscher2022beyond} showed that pruning static datasets using metrics like margin scores can improve scaling laws by retaining the most informative examples, albeit in a non-adaptive manner.
 
\textbf{Challenges and risks of synthetic data.}
The challenges of training models on synthetic data, have gained significant attention. \citet{dohmatob2024strong,dohmatob2024tale} studied “model collapse”, a phenomenon where iterative training on synthetic data degrades performance. 
They emphasize that data verification mechanisms can mitigate this risk and enable scaling with synthetic data. Similarly, our framework by generating informative examples through a dynamic loop, improves sample efficiency.

%% file: 7-discussion.tex
\section{Conclusion}
We introduced Deliberate Practice for Synthetic Data Generation, a framework that improves scaling laws by dynamically generating challenging and informative training examples. Unlike traditional methods that rely on static datasets, our approach approximates generating data directly from a pruned distribution, reducing inefficiencies and ensuring models continuously training on informative samples. We provided theoretical insights into the benefits of training on pruned distributions and empirically demonstrated that our method significantly improves performance while requiring fewer training iterations. Our results on ImageNet-100 and ImageNet-1K show that Deliberate Practice achieves superior accuracy with far less data and compute, outperforming previous state-of-the-art. Our work highlights the potential of structured synthetic data generation in advancing efficient and adaptive learning.

%% file: 8-appendix.tex
\section{Further Theoretical Analysis and proofs}\label{app:theory}
\subsection{The Unregularized Regime}

We now consider our theory in the limit  $\lambda \to 0^+$. Thus, the parameter vector for the classifier is the least-squares estimate for $w_0$, i.e $\hat w=\hat w_{LS} = {X'}^\dagger Y'$. We have the following important corollary to Theorem \ref{thm:main}.

\begin{corollary}
It holds that 
\begin{eqnarray}
E_{test}(\hat w) \to \Phi(-\frac{a}{\sqrt{b-a^2}})\text{ in the limit }n,d \to \infty,\,d/n \to \phi,\,\lambda\to 0^+,
\end{eqnarray}
where the constants $a$ and $b$ are given as follows:

(A) If $\phi < p$, then
\begin{align}
 a &:= \beta \sqrt{\frac{2}{\pi}} \frac{r_0}{p-\phi},\quad b := \frac{p^2\phi + \beta^2\cdot\left(r'_0-2\phi r_0\right)}{(p-\phi)^3},\\
 \text{with }
r_0 &:= 1-\rho^2+\rho^2\cdot p/\gamma,\quad
r'_0 := p\cdot \left(1-\rho^2 + \rho^2\cdot ((p-\phi)p/\gamma^2+\phi/\gamma)\right).
\end{align}

(B) If $\phi > p$, then
\begin{align}
a &:= \beta\sqrt{\frac{2}{\pi}}c_0r_0,\quad b := c_0\cdot\left(p\phi - \beta^2 r_0\right),\\
\text{with }c_0 &:=1-p/\phi,\quad r_0 := 1-\rho^2 + \frac{\rho^2}{\gamma/\phi + c_0}.
\end{align}
\label{cor:ridgeless}
\end{corollary}
The result is empirically verified in Figure \ref{fig:figcool}(a).

\begin{figure}[!h]
    \centering
    \begin{subfigure}{0.6\linewidth}
        \centering
        \includegraphics[width=\linewidth]{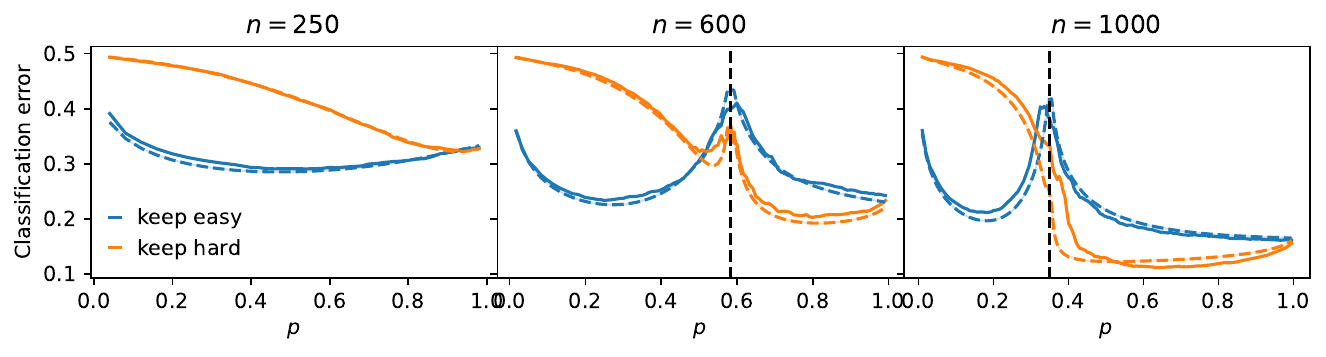}
        \caption{Regularization parameter $\lambda=10^{-6}$.}
        \label{fig:subfig1}
    \end{subfigure}
    \hfill
    \begin{subfigure}{0.6\linewidth}
        \centering
        \includegraphics[width=\linewidth]{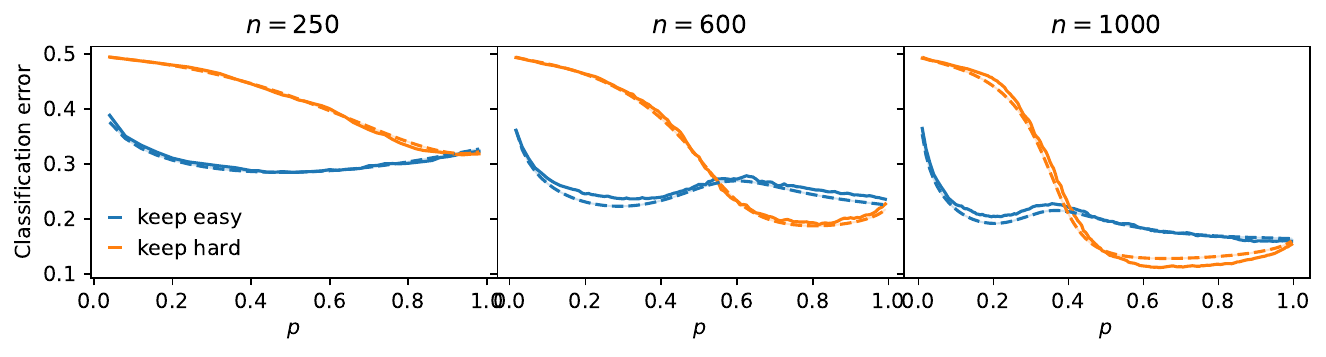}
        \caption{Regularization parameter $\lambda=10^{-2}$.}
        \label{fig:subfig2}
    \end{subfigure}

    \caption{\textbf{Empirical verification of Theorem \ref{thm:main} and Corollary \ref{cor:ridgeless}.} For this experiment, the input dimension is $d=350$, and each subplot corresponds to a different value of the original sample size $n$. The experiment for $\lambda=10^{-6}$ is a proxy for the unregularized case $\lambda\to 0^+$. Solid lines correspond to observed values of the test error $E_{\text{test}}(\hat{w})$, while broken lines are the theoretical prediction of Theorem \ref{thm:main} (\textbf{bottom row}) and Corollary \ref{cor:ridgeless} (\textbf{top row}). Notice the excellent match between the experimental results and our theory. Also, observe the multiple-descent patterns, reminiscent of a non-trivial effect of different pruning strategies in different regimes of the pruned training dataset size $n_0=np$; the vertical line corresponds to an interpolation threshold at $p=\phi$, i.e., $n_0=d$.}
    \label{fig:figcool}
\end{figure}

\subsection{Some Important Examples of Pruning Strategies}
\paragraph{Keep Hard Examples (KH).} Consider the case where the pruning strategy is given by $q_i = q_{KH}(x_i^\top w_s)$ for all $i$, where
\begin{eqnarray}
q_{KH}(t) := 1[|t| \le \xi] =  \begin{cases}
    1,&\mbox{ if }|t| \le \xi,\\
    0,&\mbox{ else,}
    \end{cases}
\end{eqnarray}
for some $\xi \ge 0$. Define $\alpha:=\xi/\|w_s\|$. We have explicit formula for the constants $\beta$ and $\tilde \beta$ appearing in Theorem \ref{thm:main}. Viz,
\begin{lemma}
With $\tau:=\rho/\sqrt{1-\rho^2}$, $\epsilon_1 := 2\Phi(\alpha/\sqrt{1-\rho^2})-1$, and $\epsilon_2 := 2\Phi(\tau \alpha)-1$, it holds that
\begin{align}
    \tilde \beta(q_{KH}) &= 2(\rho\varphi(0)\epsilon_1-\varphi(\alpha)\epsilon_2),\quad
    \beta(q_{KH}) = 2\varphi(0)\sqrt{1-\rho^2}\cdot \epsilon_1.
\end{align}
\label{lm:KH}
\end{lemma}

\paragraph{Example 2: Keep Easy Examples (KE).} Here, the pruning strategy is $q_i = q_{KE}(x_i^\top w_s)$, where
\begin{eqnarray}
 q_{KE}(t) := 1[|t| > \xi] = \begin{cases}
    0,&\mbox{ if }|t| \le \xi,\\
    1,&\mbox{ else.}
\end{cases}
\end{eqnarray}
\begin{lemma}
\label{lm:KE}
With $\tau:=\rho/\sqrt{1-\rho^2}$, $\epsilon_1 := 2(1-\Phi(\alpha/\sqrt{1-\rho^2}))$, $\epsilon_2 := 2\Phi(\tau \alpha)-1$, it holds that
\begin{align}
    \tilde \beta(q_{KE})&= 2(\rho\varphi(0)\epsilon_1+\varphi(\alpha)\epsilon_2),\quad \beta(q_{KE}) = 2\varphi(0)\sqrt{1-\rho^2}\cdot \epsilon_1.
\end{align}
\end{lemma}

\paragraph{Example 3: Interpolation between Keep Hard and Keep Easy Strategies.} Consider the following pruning strategy proposed in \cite{kolossov2024towards}
\begin{eqnarray}
    q(t) \propto \sigma(t)^\omega (1-\sigma(t))^\omega,
\end{eqnarray}
for some tuning parameter $\omega$. Here, $\sigma$ is the sigmoid function. We can associate $q(x_i^\top w_s)$ with the probability the auxiliary classifier $x \mapsto sign(x^\top w_s)$ assigns to an example $x_i$. Thus, positive values of $\omega$ correspond to keeping examples considered uncertain (i.e hard) by this classifier, while negative values correspond to examples considered easy.

\subsection{Main Ingredients of Proofs}
\subsubsection{Deterministic Equivalent for the Resolvent Matrix $R$}
\begin{definition}[Deterministic Equivalents]
    Given a sequence of random $N \times N$ matrices $(R_N)_N$, a deterministic equivalent thereof is a sequence of deterministic  $N \times N$  matrices $(\overline R_N)_N$ such that
    \begin{eqnarray}
    \trace A_N (R_N -\overline R_N) \overset{a.s}{\to} 0,
    \end{eqnarray}
    for all sequences of $N \times N$ matrices $(A_N)_N$ with bounded Frobenious norm.
\end{definition}

  Let $\Pi$ (resp. $\Pi_\perp=I_d-\Pi$) be the projection onto the span (resp. orthogonal complement of the span) of $w_s$. Define the following auxiliary vectors and scalars
\begin{align}
v &= \Sigma^{1/2}w_s,\quad v_1 = \frac{v^\top w_s}{\|w_s\|},\quad v_\perp = \Pi_\perp v.
\label{eq:projproj}
\end{align}
Note that $v_\perp$ is $(d-1)$-dimensional and $\|v_\perp\| = \sqrt{\|v\|^2-v_1^2}$.

Henceforth we make the replacement $z=-\lambda<0$, so that the resolvent matrix $R$ now writes
\begin{eqnarray}
    R = R(z) := (X^\top D X/n-zI_d)^{-1}.
\end{eqnarray}
 Let $\delta(z)$ be the unique positive solution to the fixed-point equation
\begin{align}
m(z) &= d^{-1}\trace\bar R_b(z),\quad \delta(z) = n^{-1}\trace\Sigma \bar R_b(z),\\
\bar R_b(z) &= \left(\mathbb E_{x \sim \mathcal N(0,\Sigma)}\,\left[\frac{q(x^\top w_s)}{1+q(x^\top w_s)\delta(z)}\right]\Sigma-zI_d\right)^{-1}.
\end{align}

Note that the inner expectation evaluates to
$$
\mathbb E_{x \sim \mathcal N(0,\Sigma)}\,\left[\frac{q(x^\top w_s)}{1+q(x^\top w_s)\delta(z)}\right] = \frac{p}{1+\delta(z)}=:t(z),
$$
and so $\bar R_b(z) = (t(z)\Sigma-z I_d)^{-1}$. Observe that $\bar R_b(z)(t(z)\Sigma-zI_d) = I_d$, and so $t(z)\Sigma\bar R_b(z) = I_d+z\bar R_b(z)$. We deduce that
\begin{align*}
t(z)\delta(z) &= n^{-1}\trace t(z)\Sigma \bar R_b(z) = n^{-1}\trace (I_d+z \bar R_b(z))=\phi\cdot\left(1+zm(z)\right).
\end{align*}
Thus the equations defining $m(z)$ and $\delta(z)$ can be rewritten as 
\begin{align}
m(z) &=  d^{-1}\trace (t(z)\Sigma-z I_d)^{-1},\\
t(z) &= \frac{p}{1+\delta(z)},\\
\phi\cdot(1+zm(z)) &= t(z)\delta(z) = t(z)\left(\frac{p}{t(z)}-1\right)=p-t(z).
\end{align}
Solving for $\phi z m(z)$ in terms of $t(z)$ in the last equation gives
$$
\phi zm(z) = \frac{p\delta(z)}{1+\delta(z)}-\phi = p-\phi - \frac{p}{1+\delta(z)}=p-\phi-t(z).
$$
Plugging this into the first equation gives the following fixed-point equation for $t(z)$
\begin{equation}
p-\phi-t(z) = zn^{-1}\trace(t(z)\Sigma-zI_d)^{-1}.
\label{eq:tz}
\end{equation}
The following result shows that $\bar R$ is a deterministic equivalent for $R$.
\begin{proposition}
Recall the function $t(z)$ as the unique positive solution to the equation \eqref{eq:tz}. Then,
\begin{align}
    R &\simeq \bar R,\text{ with }\bar R = \Sigma^{-1/2}(\bar m(z)\Pi_\perp + \tilde m(z)\Pi)\Sigma^{-1/2},\\
    \text{where }\bar m(z) &= \frac{1}{t(z)-z},\quad \tilde m(z) = \frac{1}{s(z)-z},\quad s(z) = \frac{\gamma}{1+\delta(z)}=(\gamma/p)t(z),\\
\gamma &:= \mathbb E [h(v_1 G_1 + \|v_\perp\|G_\perp )G_1 ^2],\\
h(x) &:= \frac{q(x)}{1+q(x)\delta(z)},\quad (G_1 ,G_\perp) \sim \mathcal N(0,I_2).
\end{align}
\label{prop:deterministic-equivalents}
\end{proposition}

\subsection{Isotropic Case}
Consider the special case where the covariance matrix is $\Sigma=I_d$. It is not hard to see that we must have $\bar m(z) \equiv m(z) \equiv \delta(z)/\phi$.
Let us now compute $m(z)$.

\begin{lemma}
For every $z=-\lambda<0$, $m(z)$ is given by formula \eqref{eq:meq-mp}.
\label{lm:mp}
\end{lemma}
\begin{proof}
Indeed, observe that in the isotropic case the equation \eqref{eq:tz} reduces to $p-\phi-t(z) = \phi z / (t(z)-z)$, or equivalently
\begin{eqnarray*}
    0 = \phi z + (t(z)-p+\phi)(t(z)-z) = t(z)^2-(p-\phi+z)t(z) + pz.
\end{eqnarray*}
The discriminant of this quadratic equation evaluates to
\begin{align*}
(p-\phi+z)^2 -4pz &= (p-\phi-z + 2z)^2 - 4pz\\
&= (p-\phi-z)^2+4z^2+4z(p-\phi-z) - 4pz\\
&= (p-\phi-z)^2-4\phi z,
\end{align*}
and so because $z=-\lambda<0$, the positive solution is
\begin{eqnarray}
    t(z) = \frac{p-\phi + z + \sqrt{(p-\phi-z)^2 - 4\phi z}}{2}.
\end{eqnarray}
We deduce that
\begin{align*}
    m(z) &= \frac{1}{t(z)-z} = \left(\frac{p-\phi - z + \sqrt{(p-\phi-z)^2 - 4\phi z}}{2}\right)^{-1}\\
    &=2 \cdot \frac{p-\phi-z-\sqrt{(p-\phi-z)^2-4\phi z}}{(p-\phi-z)-((p-\phi-z)^2-4\phi z)}\\
    &= \frac{p-\phi-z-\sqrt{(p-\phi-z)^2-4\phi z}}{2\phi z},
\end{align*}
which is precisely the claimed formula given in \eqref{eq:meq-mp}.
\end{proof}

The following result then follows directly from Proposition \ref{prop:deterministic-equivalents}.
\begin{corollary}
In the isotropic setting, we have the following deterministic equivalents:
\begin{align}
R &\simeq \bar R,\text{ with }\bar R = m(z)\Pi_\perp +s(z) \Pi,\\
R^2 &\simeq m'(z)\Pi_\perp +  \tilde m'(z)\Pi.
\end{align}
where $\tilde m(z) := 1/(s(z)-z)$, $s(z) = \gamma/(1+\phi m(z))$, and $\gamma \ge 0$ is as given in \eqref{eq:constants}.

\label{corr:deterministic-equivalents-isotropic}
\end{corollary}

\begin{eqnarray}
    \begin{split}
\rho = \frac{w_s^\top w_0}{\|w_s\|\|w_0\|},\,
\beta &:= \mathbb E\,[q(\|w_s\| G_2)|G_1|],\,
\gamma := \mathbb E\,[q(\|w_s\|G_1)G_1^2],
    \end{split}
\label{eq:constants}
\end{eqnarray}

\subsection{Test Error Representation ("Scaling Laws")}
We are now ready to state our main theoretical results, which is a generalization of Theorem \ref{thm:main}.

\begin{remark}
For simplicity of presentation, all our theoretical results only consider symmetric pruning strategies for which $q(-t) \equiv q(t)$. This includes the "keep hard" and "keep easy" pruning strategies considered in \citep{sorscher2022beyond}.
\end{remark}

\begin{proposition}
    For a random test point $(x,y) \sim P$ independent of training data, it holds that $yx^\top \hat w \overset{\mathcal L}{\to} N(m, \nu-m^2)$ in the limit \eqref{eq:proportionate}, where
\begin{align}
m &:= \frac{m_0}{1+\delta},\quad m_0 := \mu^\top \bar R\, c\\
\nu &:= \frac{\nu_0}{(1+\delta)^2},\quad \nu_0 :=\frac{p}{n}\trace\Sigma \Sigma' + c^\top \Sigma' c-\frac{2c^\top\bar R c}{1+\delta}\frac{1}{n}\trace \Sigma \Sigma',\\
\text{with }c &:= \mathbb E_{(x,y) \sim P'}[q(x^\top w_s)yx],\quad \Sigma' := \mathbb E\,[R\Sigma R].
\end{align}

Consequently, the limiting test error of $\hat w$ is given by 
\begin{eqnarray}
E_{test}(\hat w) \to \Phi\left(-\frac{m_0}{\sqrt{\nu_0-m_0^2}}\right).
\end{eqnarray}
\label{prop:testrep}
\end{proposition}

\subsection{Proof of Proposition \ref{prop:testrep}}
The proof follows standard \citep{Couillet_Liao_2022,Firdoussi2024} "leave-one-out" techniques which are now standard for analyses based on random matrix theory.

We start with the Woodbury identity tells us that
\begin{align*}
Rx_i &= (X^\top D X/n + \lambda I_d)^{-1}x_i = (n^{-1}\sum_{j=1}^nq_j x_jx_j^\top + \lambda I_d)^{-1}x_i\\
&= (R_{-i}^{-1} + q_ix_ix_i^\top/n)^{-1}x_i= \frac{R_{-i}x_i}{1+q_ix_i^\top R_{-i} x_i/n},
\end{align*}
where $R_{-i} := (n^{-1}\sum_{j \ne i}q_j x_jx_j^\top + \lambda I_d)^{-1}$ is a version of the resolvent matrix constructed without the $i$th data point. This "leave-one-out" trick is well-known in random matrix theory calculations. 

On the other hand $q_i x_i^\top R_{-i} x_i/n$ concentrates around its mean which is
\begin{align*}
\mathbb E\,[q_i x_i^\top R_{-i} x_i/n] &= \trace\left(\mathbb E[q_i x_ix_i^\top] R_{-i}/n\right) = \frac{\alpha}{n}\trace \Sigma R_{-i} \simeq \delta,\\
\text{with }\delta &:= \frac{p}{n}\trace\Sigma\bar R,\quad p := \mathbb E[q_i].
\end{align*}
Therefore, we have the following identities holding for every $i,j \in [n]$ with $i \ne j$:
\begin{align}
Rx_i &\simeq \frac{R_{-i}x_i}{1+\delta},\\
R_{-i} &\simeq R_{-ij}-\frac{R_{-ij}x_jx_j^\top R_{-ij}}{1+\delta}.
\end{align}

Now, let $x$ be a random test point from class $y$, independent of training data. Following a route similar to \citep{Firdoussi2024}, we shall compute the first two moments of the margin $yx^\top \hat w$. First observe that
\begin{align}
yx^\top \hat w &= \frac{1}{n}\sum_{i=1}^n q_i y_i yx^\top R x_i = \frac{1}{n}\sum_{i=1}^n q_i y_i yx^\top R x_i\nonumber\\
&=\frac{1}{(1+\delta)n}\sum_{i=1}^n q_i y_i yx^\top R_{-i} x_i
\label{eq:pred-decomp}
\end{align}

\subsection{First Moment of Test Margin}
From \eqref{eq:pred-decomp}, one computes for a random test point $(x,y) \sim P$,
\begin{align*}
\mathbb E\,[yx^\top \hat w] &= \frac{1}{(1+\delta)n}\sum_{i=1}^n \mathbb E\,[q_i y_i yx^\top R_{-i} x_i]\\
&= \frac{1}{(1+\delta)n}\sum_{i=1}^n \mathbb E\,[yx]^\top \mathbb E\,[R_{-i}]\mathbb E[q_i y_i x_i]\\
&= \frac{1}{(1+\delta)}\mu^\top\bar R\frac{1}{n}\sum_{i=1}^n \mathbb E[q_i y_ix_i]\\
&= \frac{1}{(1+\delta)}\mu^\top\bar R\,c,\\
\text{where }\mu &= \mathbb E_{(x,y)}\,[yx],\quad  c := \mathbb E_{(x,y)}\,[q(x^\top w_s) y x].
\end{align*}
The following result computes the mean vectors $\mu$ and $c$.
\begin{lemma}
Let $\rho \in [-1,1]$ be the cosine of the angle between $\bar w_s:=\Sigma^{1/2}w_s$ and $\bar w_0:=\Sigma^{1/2}w_0$. Let $u$ be the unit-vector in the direction of $\bar w_s$  and let $v$ be its completion to an orthonormal basis for the span of $\bar w_s$ and $\bar w_0$ (if $\bar w_s$ and $\bar w_0$ are parallel, i.e if $\rho=\pm 1$, we simply set $v=0$).
\begin{eqnarray}
 \mu  :=  \mathbb E_{(x,y) \sim P}[yx],\quad c :=  \mathbb E_{(x,y) \sim P}[q(x^\top w_s)y x]
\end{eqnarray}
Then, $\mu=\sqrt{2/\pi}\cdot \Sigma w_0/\|w_0\|_\Sigma$, and $c = \tilde \beta u + \beta v$, where
\begin{align}
    \tilde \beta = \beta_1 &: = 2\mathbb E\left[q(\|\bar w_s\|G)\Phi\left(\tau G\right)G\right],\quad \beta = \beta_2 := 2\mathbb E\left[q(\|\bar w_s\|G)\varphi(\tau G)\right],\quad \text{with }G\sim \mathcal N(0,1).
\end{align}

In particular, when $\rho = \pm 1$ (i.e pruning along the data generator), 
\begin{align}
    \beta_1 = \mathbb E[q(\|\bar w_s\|G)|G|],\quad \beta_2=0. 
\end{align}

\label{lm:means}
\end{lemma}

\subsection{Second Moment of Test Margin $yx^\top \hat w$}
Squaring \eqref{eq:pred-decomp} gives
\begin{align*}
(yx^\top \hat w)^2 &= \frac{1}{(1+\delta)^2n^2}\sum_{i=1}^n q_i\cdot (x^\top R_{-i} x_i)^2 + \frac{1}{(1+\delta)^2n^2}\sum_{i \ne j} q_iq_jy_iy_j (x^\top R_{-i} x_i)(x^\top R_{-j} x_j)
\end{align*}
For the expectation first some, note that
\begin{align*}
\frac{1}{n}\mathbb E\,[q_i \cdot (x^\top R_{-i} x_i)^2] = \frac{1}{n}\mathbb E[q_i x^\top R_{-i} x_ix_i^\top R_{-i}x]&= \frac{1}{n}\trace\left(\mathbb E\,[xx^\top]\mathbb E\,[q_i R_{-i}x_ix_i^\top R_{-i}]\right)= \frac{p}{n}\trace\Sigma \Sigma',
\end{align*}
with $\Sigma':= \mathbb E[R\Sigma R]$. We deduce that
\begin{align*}
\mathbb E\,\frac{1}{(1+\delta)^2n^2}\sum_{i=1}^n q_i\cdot (x^\top R_{-i} x_i)^2 &= \frac{1}{(1+\delta)^2}\frac{p}{n}\trace \Sigma \mathbb E\,[R \Sigma R]\nonumber\\
&= \frac{p}{(1+\delta)^2}\cdot \begin{cases}
n^{-1}\trace \mathbb E\,[R^2] \Sigma,&\mbox{ if isotropic},\\
\text{hard life!},&\mbox{ otherwise.}
\end{cases}
\end{align*}

Now, let $i,j \in [n]$ with $i \ne j$. One computes
\begin{align*}
\mathbb E\,[q_iq_jy_iy_j \cdot (x^\top R_{-i} x_i)(x^\top R_{-j} x_j)]
&= \frac{1}{1+\delta}\mathbb E\,\left[q_i q_j y_i y_jx_i^\top T_{ij}\Sigma T_{ji}x_j\right],\\
&= \frac{1}{1+\delta}(A_1-A_2-A_3+A_4),\\
\text{where }
T_{ij} &:= R_{-ij}-S_{ij}/n,\\
S_{ij} &:= \frac{R_{-ij}x_jx_j^\top R_{-ij}}{1+\delta},\\
A_1 &:= \mathbb E\,[q_iq_jy_iy_jx_i^\top R_{-ij}\Sigma R_{-ij}x_j],\\
A_2 &:= \frac{1}{(1+\delta)n}\mathbb E\,[q_iq_jy_iy_jx_i^\top S_{ij}\Sigma R_{-ij}x_j],\\
A_3 &:= \frac{1}{(1+\delta)n}\mathbb E\,[q_iq_jy_iy_jx_i^\top R_{-ij}\Sigma S_{ji}x_j],\\
A_4 &:= \frac{1}{(1+\delta)^2n^2}\mathbb E\,[q_iq_jy_i y_jx_i^\top S_{ij}\Sigma S_{ji}x_j]
\end{align*}
We now compute the terms $A_1,A_2,A_3,A_4$.
\begin{align*}
A_1 &= \mathbb E\,[q_iq_jy_iy_jx_i^\top R_{-ij}\Sigma R_{-ij}x_j] = \mathbb E\,[q_iq_jy_iy_jx_i^\top R\Sigma R x_j]\\
&= \trace\left(\mathbb E\,[(q_j y_j x_j)(q_iy_ix_i)^\top]\mathbb E\,[R\Sigma R]\right)= c^\top \Sigma' c,\\
\text{where }
\Sigma' &:= \mathbb E[R\Sigma R].
\end{align*}

Similarly, $A_3=A_2$ with
\begin{align*}
A_2 = \mathbb E\,[q_iq_jy_iy_jx_i^\top S_{ij}\Sigma R_{-ij}x_j] &= \frac{1}{(1+\delta)n}\mathbb E\,[q_iq_jy_iy_jx_i^\top R_{-ij}x_jx_j^\top R_{-ij}\Sigma R_{-ij}x_j]\\
&= \frac{1}{(1+\delta)n}\trace\left(\mathbb E\,[q_iq_jy_iy_jx_jx_i^\top R_{-ij} x_j x_j^\top]\mathbb E\,[R_{-ij}\Sigma R_{-ij}]\right)
\end{align*}
Now, computes 
\begin{align*}
\mathbb E\,[q_iy_iq_jy_jx_i^\top R_{-ij}x_j] &= \mathbb E\,[(q_iy_ix_i)^\top R_{-ij}(q_jy_jx_j)]= c^\top\mathbb E\,[R_{-ij}] c \simeq c^\top\mathbb E\,[R] c \simeq c^\top\bar R c,\\
\mathbb E[R_{-ij}\Sigma R_{-ij}] &\simeq \mathbb E[R\Sigma R] =: \Sigma',
\end{align*}
 from which it follows that
\begin{eqnarray}
A_3 = A_2 \simeq \frac{c^\top \bar R c}{1+\delta}\frac{1}{n}\trace \Sigma \Sigma'.
\end{eqnarray}

Finally, it is easy to show that $A_4=O(1/n)=o(1)$.

Putting things together gives the result. \qed

\subsection{Proof of Lemma \ref{lm:means}}

Observe that by instead considering $\Sigma^{-1/2}\mu$, $\Sigma^{-1/2}c$, and defining $v:=\Sigma^{1/2}w_s$ and  $u:=\Sigma^{1/2}w_0$ when computing $\mu$, and then  $u=\Sigma^{1/2}w_0$ when computing $c$, we reduce the problem to the isotropic case $x \sim \mathcal N(0,I_d)$.

So let $u=\Sigma^{1/2}w_0$, and WLOG, assume $u$ is aligned with the first canonical axis in $\mathbb R^d$, i.e $u=\|u\| e_1$. Write $x=(x_1,x_\perp)$ and $v = (v_1,v_\perp)$, where $x_\perp := \sum_{j=2}^d x_j e_j  \in \mathbb R^{d-1}$, and $v_\perp := \sum_{j=2}^d v_j e_j  \in \mathbb R^{d-1}$. It is clear that $x^\top u = \|u\|x_1$, and $x^\top v = v_1 x_1 + g$, where  $g=x_\perp^\top v_\perp$. Furthermore, $x_1$ and $g$ are independent with distributions $\mathcal N(0,1)$ and $\mathcal N(0,\|v_\perp\|^2)$ respectively. It follows that
    \begin{align*}
        \Sigma^{-1/2}\mu &= \mathbb E\,[sign(x^\top u)x] = \mathbb E\,[sign(\|u\|x_1)x_1]e_1 = \mathbb E\,[|x_1|]e_1\\
        &= \sqrt{\frac{2}{\pi}}e_1 = \sqrt{\frac{2}{\pi}}\frac{u}{\|u\|} = \sqrt{\frac{2}{\pi}}\frac{\Sigma^{1/2}w_0}{\|w_0\|_\Sigma},
    \end{align*}
    from which we deduce the prescribed formula for the vector $\mu$. This proves the first part of the claim.

We now establish the formula $c=\beta_1 u + \beta_2 v$. The proof for the formula for $\mu$ follows a similar (but simpler) path.

Observe that by instead considering $\Sigma^{-1/2}c$, we reduce the problem to the isotropic case $x \sim \mathcal N(0,I_d)$. We can explicitly write
\begin{align}
    u = \frac{\bar w_s}{\|\bar w_s\|},\quad v = \frac{\Pi^\perp \bar w_0}{\|\Pi^\perp\bar w_0\|},\quad \rho = \frac{\bar w_s^\top \bar w_0}{\|\bar w_s\|\|\bar w_0\|},
\end{align}
where $\Pi=uu^\top$ and $\Pi^\perp=I_d-\Pi$. One can decompose $x=G_1u + G_2v + G_\perp$ and $\bar w_0 = c_1u + c_2v + c_\perp$
\begin{align}
    G_1 &:= x^\top u, \quad G_2 := x^\top v,\quad G_\perp := P^\perp x,\\
    c_1 &:= w_0^\top u,\quad c_2:= x^\top v,\quad c_\perp :=\quad P^\perp \Sigma^{1/2}w_0,
\end{align}
where $P$ is the projector onto the span of $u$ and $v$.
Note that $G_1$, $G_2$, and $G_\perp$ forms a set of independent random variables. Moreover, $G_1$ and $G_2$ have distribution $\mathcal N(0,1)$, while $G_\perp$ has distribution $\mathcal N(0,I_{d-2})$. We obtain
\begin{align}
    \mathbb E[q(x^\top w_s)sign(x^\top w_0)x] &= \mathbb E\,[q(x^\top w_s)sign(x^\top w_0)x] = \mathbb E\,[q(x^\top w_s)sign(x^\top w_0)x]\\
    &= \mathbb E\,[q(\|w_s\|G_1)sign(c_1 G_1 + c_2G_2)G_1]\cdot u\\
    &\quad + \mathbb E\,[q(\|w_s\|G_1)sign(c_1 G_1 + c_2G_2)G_2]\cdot v\\
    &\quad + \mathbb E\,[q(\|w_s\|G_1)sign(c_1 G_1 + c_2G_2)G_\perp].
\end{align}
Now, due independence, the third term decomposes as
$$
\mathbb E\,[q(\|w_s\|_\Sigma \cdot G_1)sign(c_1 G_1 + c_2G_2)]\cdot \mathbb E\,[G_\perp]=0.
$$
We deduce that
\begin{align*}
\mathbb E[q(x^\top w_s)sign(x^\top w_0)x] &= \beta_1 u + \beta_2 v,\\
\end{align*}
where $\beta_1$ and $\beta_2$ are as specified in the lemma
and we have used the fact that
$$
c_1/\|\bar w_0\| = \rho,\quad c_2/\|\bar w_0\| = \sqrt{1-\rho^2}.
$$

In particular, if $\rho = \pm 1$ (meaning that $w_0$ and $w_s$ are parallel), then 
\begin{align}
    \beta_k &= \mathbb E\left[sign(\pm G_1)q(\|\bar w_s\|\cdot G_1)G_k\right] = \begin{cases}
        \pm \beta,&\mbox{ if }k=1,\\
        0,&\mbox{ otherwise.}
    \end{cases}
\end{align}

We now compute the coefficients $\beta_1$ and $\beta_2$. Observe that thanks to Lemma \ref{lm:angel}, one has
\begin{align*}
\mathbb E[sign(G_3) \mid G_1]
&=  \mathbb E[sign(\rho G_1 + \sqrt{1-\rho^2}G_2)\mid G_1] = 2\Phi\left(\tau G_1\right)-1,\\
\mathbb E[sign(G_3)G_2) \mid G_1] &=  \mathbb E[sign(\rho G_1 + \sqrt{1-\rho^2}G_2)G_2 \mid G_1] = 2\varphi(\tau G_1).
\end{align*}
Therefore, with $r := \|\bar w_s\|$, we have
\begin{align*}
    \beta_1 &:= \mathbb E[q(rG_1) sign(G_3) G_1] = 2\mathbb E\left[q(rG_1)\Phi\left(\tau G_1\right)G_1\right] - \mathbb E\left[q(rG_1)G_1 \right]=2\mathbb E\left[q(rG_1)\Phi\left(\tau G_1\right)G_1\right],\\
    \beta_2 &:= \mathbb E[q(rG_1) sign(G_3) G_2] = 2\mathbb E\left[q(rG_1)\varphi(\tau G_1)\right],
\end{align*}
where we have used the oddness of the function $t \mapsto tq(rt)$ in the last equation on the first line. \qed

\begin{lemma}
\label{lm:angel}
    Let $G \sim \mathcal N(0,1)$, and let $a,b \in \mathbb R$ with $a \ne 0$. Then,
    \begin{align}
        \mathbb E[sign(aG+b)] &= 2\Phi(b/|a|)-1,\quad 
        \mathbb E[sign(aG+b)G] = 2\varphi(b/a).
    \end{align}
Furthermore, it holds that
\begin{align}
    \lim_{a \to 0} \mathbb E[sign(aG+b)] &= sign(b),\quad \lim_{a \to 0} \mathbb E[sign(aG+b)G] = 0. 
\end{align}
\end{lemma}
\begin{proof}
    Indeed, one computes
    \begin{align*}
    \mathbb E[sign(aG+b)] &= \mathbb P(aG+b>0)-\mathbb P(aG+b<0)=2\mathbb P(aG>-b)-1\\
    &= \begin{cases}
        2\mathbb P(G>-b/a) - 1 = 2\Phi(b/a)-1,&\mbox{ if }a>0,\\
        2\mathbb P(G<-b/a)-1=2\Phi(-b/a)-1,&\mbox{ if }a<0.
    \end{cases}
    \end{align*}
    We deduce that $\mathbb E[sign(aG+b)]=2\Phi(b/|a|)-1$ as claimed.
\end{proof}

\section{Proof of Lemma \ref{lm:KH} and Lemma \ref{lm:KE}}
\textbf {"Keep Hard" Examples (Lemma \ref{lm:KH}).}
Let $b=\tau$, $t=\sqrt{1+b^2}=\sqrt{1+\tau^2}=1/\sqrt{1-\rho^2}$. Using Lemma \ref{lm:means} and standard formulae\footnote{For example, see Wikipedia \url{https://en.wikipedia.org/wiki/List_of_integrals_of_Gaussian_functions}.} for the anti-derivative of the function $z \mapsto z\varphi(bz)\varphi(z)$
\begin{align*}
\beta = \beta_2 &= 2\mathbb E\left[q(rG)\varphi(\tau G)\right] = 2\int_{-\alpha}^\alpha \varphi(\tau z)\varphi(z)\mathrm dz = \frac{2}{t}\varphi(0)\Phi(tz)\bigg]_{z=-\alpha}^\alpha\\
&= 2\sqrt{1-\rho^2} \varphi(0)\left(2\Phi(\alpha/\sqrt{1-\rho^2})-1\right) =2\varphi(0)\sqrt{1-\rho^2}\epsilon_2.
\end{align*}
On the other hand, we have $\tilde\beta = \beta_1 = 2\mathbb E\left[q(rG)\Phi (\tau G)G\right] = 2\int_{-\alpha}^\alpha z\Phi(\tau z)\varphi(z)\mathrm d z$ with
\begin{align*}
\int_{-\alpha}^\alpha z\Phi(\tau z)\varphi(z)\mathrm d z 
&= (b/t)\varphi(0)\Phi(tz) -\varphi(z)\Phi(bz)\bigg]_{z=-\alpha}^\alpha \\
&=(b/t)\varphi(0)(2\Phi(t\alpha)-1) - \varphi(\alpha)(2\Phi(b\alpha)-1)\\
&= \rho\varphi(0)(2\Phi(\alpha/\sqrt{1-\rho^2})-1)-\varphi(\alpha)(2\Phi(\tau \alpha)-1)\\
&= \rho\varphi(0)\epsilon_1-\varphi(\alpha)\epsilon_2,
\end{align*}
which proves Lemma \ref{lm:KH}

\textbf{"Keep Easy" Examples (Lemma \ref{lm:KE}).}
Indeed, since $q_{KE} = 1-q_{KH}$, we know from the previous lemma (KH strategy) that
\begin{align*}
    \tilde \beta(q_{KE}) &= 2\mathbb E\,[q_{KE}(rG)\Phi(\tau G)G] = 2\mathbb E\,[\Phi(\tau G)G]-2\mathbb E\,[q_{KH}(rG)\Phi(\tau G)G]\\
    &= 2\mathbb E\,[\Phi(\tau G)G]-2\tilde \beta(q_{KH}) = 2(\rho\varphi(0)-\varphi(\alpha))-\tilde \beta(q_{KH})\\
    &= 2\rho\varphi(0)-2\rho \varphi(0)\epsilon_1(q_{KH})+2\varphi(\alpha)\epsilon_2(q_{KH})\\
    &= 2(\rho\varphi(0)(1-\epsilon_1(q_{KH}))+\varphi(\alpha)\epsilon_2(q_{KH}))\\
    &= 2(\rho\varphi(0)\epsilon_1 + \varphi(\alpha)\epsilon_2).
\end{align*}

The computation of $\beta_2(q_{KE})$ uses a completely analogous idea:
\begin{align*}
    \beta(q_{KE}) &= 2\mathbb E[q_{KE}(rG)\varphi(\tau G)]=2\mathbb E[\varphi(\tau G)]-2\mathbb E[q_{KH}(rG)\varphi(\tau G)]\\
    &= 2\varphi(0)\sqrt{1-\rho^2} - 2\beta(q_{KH})\\
    &= 2\left(\varphi(0)\sqrt{1-\rho^2}-\varphi(0)\sqrt{1-\rho^2}\epsilon_1(q_{KH})\right)\\
    &= 2\varphi(0)\sqrt{1-\rho^2}\left(1-\epsilon_1(q_{KH})\right)\\
    &= 2\varphi(0)\sqrt{1-\rho^2}\epsilon_1(q_{KE})
\end{align*}
This proves Lemma \ref{lm:KE}. \qed

\subsection{Proof of Proposition \ref{prop:deterministic-equivalents}}
Using Theorem 4 of Liao and Mahoney's "Hessian Eigenspectra of More Realistic Nonlinear Models"
 \url{https://arxiv.org/abs/2103.01519} and some basic manipulations, we can write
\begin{align}
R &\simeq \bar R,\\
\text{where }\bar R^{-1} &= \mathbb E_x\,\left[\frac{q}{1+q\delta(z)}(\Sigma^{1/2}\Pi_\perp\Sigma^{1/2} + \alpha\alpha^\top)\right]-zI_d,
\end{align}
where $q:=q(x^\top w_s)$ for $x \sim \mathcal N(0,\Sigma)$, $\alpha:=\Sigma^{1/2}\Pi x$. Since $q$ is Bernoulli with mean $p := \mathbb P(q=1)$, it is clear that
$$
\mathbb E_x\,\left[\frac{q}{1+q\delta(z)}\right] = \frac{p}{1+\delta(z)}:=t(z).
$$
This further gives
\begin{eqnarray}
    \begin{split}
\bar R^{-1} &= t(z)\Sigma^{1/2}\Pi_\perp \Sigma^{1/2}-zI_d+\Sigma^{1/2} \Pi K \Pi \Sigma^{1/2},\\
\text{with }K &:= \mathbb E_u\,\left[\frac{q(u^\top v)}{1+q(u^\top v)\delta(z)} uu^\top\right],
\end{split}
\label{eq:magical}
\end{eqnarray}
where $u := \Sigma^{-1/2}x \sim \mathcal N(0,I_d)$ and $v := \Sigma^{1/2}w_s$.

Now, to determine the matrix $K$, we first rewrite $u=(u_\parallel ,u_\perp)$ and $v=(v_1 ,v_\perp)$, where
\begin{align}
u_\parallel  &:= \frac{u^\top  w_s}{\|w_s\|} \in \mathbb R,\quad v_1  := \frac{v^\top w_s}{\|w_s\|} \in \mathbb R,\\
u_\perp &:= \Pi_\perp u \in \mathbb R^{d-1},\quad v_\perp := \Pi_\perp v \in \mathbb R^{d-1}.
\end{align}

The advantage of this representation is that
\begin{itemize}
\item  $u_\perp$ and $v_\perp$ are orthogonal to $w_s$.
\item $u_\parallel $ and $u_\perp$ are  statistically independent.
\item $u_\parallel $ has distribution $\mathcal N(0,1)$.
\item $u_\perp$ has distribution $\mathcal N(0,I_{d-1})$.
\end{itemize}
By symmetry of the situation, we know that
\begin{align*}
K&= s(z)  \Pi + s_\perp(z) \Pi_\perp,\\
\text{where }
s(z)  &:= \mathbb E [h(w^\top g)G_1 ^2],\quad
s_\perp(z) := \mathbb E [h(w^\top g) G_\perp^2]\\
w &:= (v_1 ,\|v_\perp\|) \in \mathbb R^2,\quad g := (G_1 ,G_\perp) \sim \mathcal N(0,I_2),\\
h(q) &:= \frac{q}{1+q\delta(z)}.
\end{align*}
Combining with \eqref{eq:magical}, we get
\begin{align}
\bar R^{-1}
&= \Sigma^{1/2}(a(z)I_d+b(z)\Pi)\Sigma^{1/2},\\
\text{where }
a(z) &= t(z) - z,\quad t(z) = \frac{p}{1+\delta(z)},\quad b(z) = s(z)-t(z).
\end{align}
Now, using the \emph{Matrix-Inversion Lemma}, one can obtain $\bar R$ from $\bar R^{-1}$ as follows:
$$
\Sigma^{1/2}\bar R\Sigma^{1/2}=(a(z)I_d+b(z)\Pi)^{-1} = \frac{1}{a(z)}\left(I_d - \frac{b(z)/a(z)}{b(z)/a(z)+1}\Pi\right) = \frac{1}{a(z)}\Pi_\perp + \frac{1}{b(z)+a(z)}\Pi.
$$
It suffices to notice that $1/(b(z)+a(z))=1/(s(z)-z) =\tilde m(z)$ and $1/a(z) = \bar m(z)$ by definition, and the result follows.

\subsection{Proof of Theorem \ref{thm:main}}
Set $z=-\lambda$. Recall from Lemma \ref{lm:means} that $\mu = \sqrt{2/\pi}w_0/\|w_0\|$ and $c = \beta_1 u + \beta_2 v$. In Theorem \ref{thm:main}, we have the identification $\beta=\beta_2$ and $\tilde \beta=\beta_1$. We know that $R \simeq \bar R = m(z)\Pi^\perp + \tilde m(z)\Pi$, where $\Pi=uu^\top$. One computes
\begin{align*}
       m_0 \simeq \mu^\top \bar R c &= \sqrt{\frac{2}{\pi}}\frac{1}{\|w_0\|}w_0^\top\left(m(z)\Pi^\perp + \tilde m(z) \Pi \right)(\beta_1 u + \beta_2 v),\\
       &= \sqrt{\frac{2}{\pi}}\frac{1}{\|w_0\|}w_0^\top\left(\beta_1 \tilde m(z)u + \beta_2m(z)v\right),\\
        \text{with }\frac{w_0^\top u}{\|w_0\|} &=\rho,\quad
        \frac{w_0^\top v}{\|w_0\|} = \frac{w_0^\top w_0/\|w_0\|-\rho\|w_0\|(u^\top w_0/\|w_0\|)}{\|w_0\|\sqrt{1-\rho^2}}\\
        &= \frac{\rho-\rho^2}{\sqrt{1-\rho^2}} =\sqrt{1-\rho^2} =:\omega/\beta_2,
\end{align*}
Putting things together gives $m_0 \simeq \sqrt{2/\pi}\cdot \left(\omega m(z) + \tilde\omega \tilde m(z)\right)$ as claimed.

Likewise, one computes
\begin{align*}
    \frac{1}{n}\trace R^2 &\simeq \frac{1}{n}\trace \left(m'(z)\Pi^\perp + \tilde m'(z)\Pi\right) \simeq \phi m'(z),\\
    c^\top \bar R c &= c^\top \left(m(z)\Pi^\perp + \tilde m(z)\Pi\right)c=(\beta_1 u + \beta_2v)^\top (\tilde m(z)\Pi + m(z)\Pi^\perp)(\beta_1 u + \beta_2 v)\\
    &= \beta_2^2 m(z) + \beta_1^2 \tilde m(z) = \beta^2 m(z) + \tilde \beta^2 \tilde m(z) =: r(z),\\
    c^\top \Sigma' c &= c^\top \mathbb E\,[R^2] c \simeq c^\top \left(m'(z)\Pi^\perp + \tilde m'(z)\Pi\right)c=\beta^2 m'(z) + \tilde \beta^2 \tilde m'(z)= r'(z),
\end{align*}
which the claimed formula for $\nu$ follows.\qed

\subsection{Proof of Corollary \ref{cor:ridgeless}}
As usual, set $z:=-\lambda<0$.

(A) For $\phi < p$, it is easy to see from formula \eqref{eq:meq-mp} and Lemma \ref{lm:primes} that in the limit $z \to 0$, one has
\begin{align*}
    m(z) &\to \frac{1}{p-\phi},\\
    \bar m(z) &\to 0,\\
    \tilde m(z) &\to \frac{p/\gamma}{p-\phi},\\
    m'(z) &\to \frac{p}{(p-\phi)^3},\\
    \bar m'(z) &\to \frac{1}{p-\phi},\\
    \tilde m'(z) &\to \frac{p/\gamma^2}{(p-\phi)^3}\left(p(p-\phi)+\phi\gamma\right)=\frac{p}{(p-\phi)^3}\left((p-\phi)p/\gamma^2+\phi/\gamma\right),\\
    \frac{ m'(z)}{1+\phi m(z)} &\to \frac{1}{(p-\phi)^2}.
\end{align*}
Furthermore, with $m_0$ and $\nu_0$  as defined in Theorem \ref{thm:main}, one computes
\begin{align*}
    r(z) &= \beta^2 m(z) + \tilde\beta^2 \tilde m(z) \to \beta^2\frac{1}{p-\phi} + \tilde \beta^2\frac{p/\gamma}{p-\phi}=\frac{r_0}{p-\phi},\\
    r'(z) &= \beta^2 m'(z) + \tilde \beta^2 \tilde m'(z) \to \beta^2\cdot \frac{p}{(p-\phi)^3} + \tilde \beta^2 \cdot\frac{p/\gamma^2}{(p-\phi)^3}(p(p-\phi)+\phi\gamma)=\frac{r'_0}{(p-\phi)^3},
\end{align*}
where $r_0$ and $r'_0$ are as defined in the claim. We deduce that $m_0/\sqrt{\nu_0-m_0^2} = a/\sqrt{b-a^2}$ and the result follows from Theorem \ref{thm:main}.

(B) Now consider the case $\phi>p$. Observe that $m_0 = \sqrt{\nu_0-m_0^2}=-zm_0/\sqrt{z^2-z^2m_0^2}$. On the other hand, from \eqref{eq:meq-mp} we know that
\begin{eqnarray}
    -zm(z) = \frac{\sqrt{(p-\phi-z)^2-4\phi z}-(p-\phi-z)}{2\phi}
\end{eqnarray}
Combining with Lemma \ref{lm:primes}, we deduce the following limits
\begin{align*}
-zm(z),z^2m'(z) &\to c_0:=1-p/\phi>0,\\
\bar m'(z) &\to \frac{p/\phi}{\phi-p},\\
-z\tilde m(z), z^2\tilde m'(z) &\to \frac{c_0}{\gamma/\phi+c_0},\\
\frac{-z m'(z)}{1+\phi m(z)} &\to \frac{1}{\phi}.
\end{align*}
Furthermore, one computes
\begin{align*}
-zr(z) &= \beta_2^2\cdot (-z m(z)) + \beta_1^2\cdot (-z \tilde m(z)) = \beta_2^2c_0 + \beta_1^2 \frac{c_0}{\gamma/\phi+c_0}=:c_0r_0,\\
z^2 r'(z) &= \beta_2^2 z^2m'(z) + \beta_1^2 z^2\tilde m(z) = \beta_2^2 c_0 + \beta_1^2 \frac{c_0}{\gamma/\phi+c_0}=c_0r_0,\\
-zm_0 &= \sqrt{2/\pi}\cdot (-zm(z)\omega-z\tilde m(z)\tilde \omega) \to \sqrt{2/\pi}c_0\cdot(\omega+\tilde\omega/(\gamma/\phi+c_0)):=a,\\
    z^2\nu_0 &= p\phi z^2 m'(z) + z^2r'(z)-2\phi\frac{-zm'(z)}{1+\phi m(z)}\cdot (-z r(z)) \\
    &\to p\phi c_0 + r_0c_0-2r_0c_0=c_0\cdot( p\phi - r_0)=:b.
\end{align*}
 We deduce that
$$
m_0/\sqrt{\nu_0-m_0^2}=-zm_0/\sqrt{z^2\nu_0-z^2m_0^2} = a/\sqrt{b-a^2},
$$
and the result follows from Theorem \ref{thm:main}.
\qed

\begin{lemma}
\label{lm:primes}
    We have the following identities:
    \begin{align*}
        m'(z) &= 
        \frac{m(z)^2}{1-(1+\bar m(z))^2\phi/p},\\
        \bar m'(z) &= \frac{p}{(z+\phi \bar m(z))^2/\bar m(z)^2-p\phi} = \frac{p}{(\phi+1/m(z))^2-p\phi}, \\
        \tilde m'(z) &= \tilde m(z)^2\left(\frac{\gamma\phi m'(z)}{(1+\phi m(z))^2} + 1\right),\\
        r'(z) &= \beta ^2 m'(z) + \tilde \beta ^2 \tilde m'(z).
    \end{align*}
\end{lemma}
\section{Additional Experimental}
\subsection{Choice of Generative Model}\label{app:choice_of_gens}
We evaluate the capabilities of four open-source large-scale pre-trained text-to-image models~\citep{rombach2022high} in a controlled setup to determine which one performs best for the image-classification task. Each synthetic image is generated with a simple prompt (\texttt{class name}). We create a dataset of size 130,000 examples and train a ViT-B model on the synthetic data. Our results (Table~\ref{tab:choice_of_gens}) show that LDM-1.5 outperforms its more recent counterparts, LDM-XL and LDM-2.1, despite being an older model. We hypothesize that this is due to the lower diversity of generations in more recent models. This finding is consistent with previous work~\citep{astolfi2024consistency}, which observed lower diversity in more recent latent diffusion models. For all of our experiments, we use LDM-1.5 as it is the best performing model.

\begin{table}[ht]
\vspace{-1em}
\centering
\caption{Study on the choice of generative model for the task of ImageNet-100 classification with synthetic data. All experiments are trained for 50k iterations and the dataset size is a static size of 130k.}\label{tab:choice_of_gens}
\vspace{0.4cm}
{\small
\begin{tabular}{lc}
\toprule
 Syn. Data Source&Real Val. Acc.\\ 
\midrule
\midrule
 LDM-1.4& 59.06  \\
 LDM-1.5&\textbf{59.24}  \\
 LDM-2.1& 55.92\\
 LDM-XL &52.8 \\
\bottomrule
\end{tabular}
}
\vspace{-1em}
\end{table}

\subsection{Ablations}\label{app:ablations}
\paragraph{$\omega=0$ vs\ \ $\omega>0$}
To understand the effect of different components of our framework, we ablate the case where data is generated through the DP framework, but with a coefficient of zero for the term $\omega$. We also report results using different values of $\omega$. See results in Table~\ref{tab:in100_ab} comparing row 1 with rows 2, 3 and 4.

\paragraph{Incremental patience}
In our experiments, setting the maximum patience value ($T_{\text{max}}$) to a fixed number resulted in the model requesting too much data when the size of the dataset was grown too big. For example, with a fixed patience of $T_\textrm{max}=7$, for an experiment with initial dataset of size 130k samples, monitoring the validation accuracy every 130k iterations, meant that in the beginning every example was seen on average 7 times before the patience reached $T_{\text{max}}$. But as we generate more examples throughout the training, with a fixed patience value, each example would not be able to be seen even at least once. When the dataset grows to be 1.3 million, each example is seen on average of 0.7 times. This resulted in the model hitting the maximum patience very often. As a result, we incrementally increase the maximum patience value as the dataset increases in size. See Table~\ref{tab:in100_ab} for the result that compares the two scenarios (comparing rows 1 and row 5). Note we found using an incremental patience to be significantly easier to tune. We often start with a patience of 1 and continue training. However, fixed patience requires more tuning depending on the size of the dataset and the number of training iterations.

\begin{table}[h]
\centering
\caption{Ablation study on ImageNet-100. Given a baseline method with DP, we modify each component of the framework one-by-one and study the effect of each change. All the experiments are trained for 50k iterations and have the same final size.}\label{tab:in100_ab}
\vspace{0.4cm}
{\small
\begin{tabular}{llllll}
\toprule
\# & $\omega$ & $T_{max}$ & Sampling & Real Val. Acc. & Real tr. Acc. \\ 
\midrule
\midrule
1   & 0.05     & inc. & uniform & 68.04  &  69.25           \\
2                         & \textbf{0}        & inc. & uniform & 61.58 &  63.09            \\
3                        & \textbf{0.03}     & inc.  & uniform & 66.70  &  68.00           \\
4                       & \textbf{0.07}     & inc.  & uniform & 66.88 &    68.66            \\
5    & 0.05     & \textbf{fixed}   & uniform & 67.22 & 70.38 \\
6           & 0.05     & inc. & \textbf{non-uni. }& 68.01 &  69.11  \\
\bottomrule
\end{tabular}
}
\end{table}

\paragraph{Dataset sampling probabilities}
One can assume that newly generated examples could be more valuable then previously generated examples. As a result, we experiment the case where every newly generated example has twice as much probability to be selected when sampling the data batch for a given iteration. We observe that having higher probability does not lead to statistically significant improvements. See results in Table~\ref{tab:in100_ab} comparing rows 1 and 7.

\subsection{Intermediate stages of reverse sampling}
In section~\ref{sec:2} we mentioned that DDIM's $x_0$ approximation is a good approximation to guide the sampling process. In this section we plot these intermediate examples which are fed to the classifier to compute the entropy of the sample and use it for guidance in the sampling process. Figure~\ref{fig:intermediate-sampling} shows that although intermediate samples are noisy, they contain the key features. 
\begin{figure}
    \centering
    \includegraphics[width=0.8\linewidth]{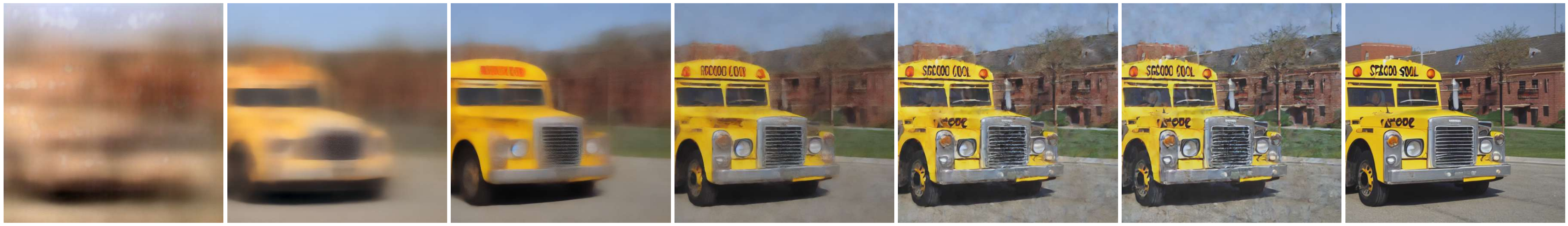}
    \caption{\textbf{Intermediate stages of reverse sampling.} Samples of the $x_0$ approximation using the DDIM sampler. While blurry, these intermediate samples provide sufficient gradients for entropy guidance, with key features like color and shape discernible even in early stages.
    }\label{fig:intermediate-sampling}
\end{figure}

\subsection{Studying the effect of $\omega$}
In this section we study the effect of dynamically generating the data without entropy guidance versus generating it uniformly from the beginning. In the first scenario, we use the DP framework, with monitoring the patience variable but using an $\omega=0$ which effectively generates with naive sampling. In the second case all the data is generated in advance and no data is added during training. As it can be see in Figure~\ref{fig:acc_dyn_vs_static}, there is close to no difference between generating all the data in advance or generating it dynamically if we allow for enough iterations for training. 

In this experiment, we evaluate on ImageNet-100 validation set and train all the models for 50,000 iterations. 

\begin{figure}
    \centering
    \includegraphics[width=0.5\linewidth]{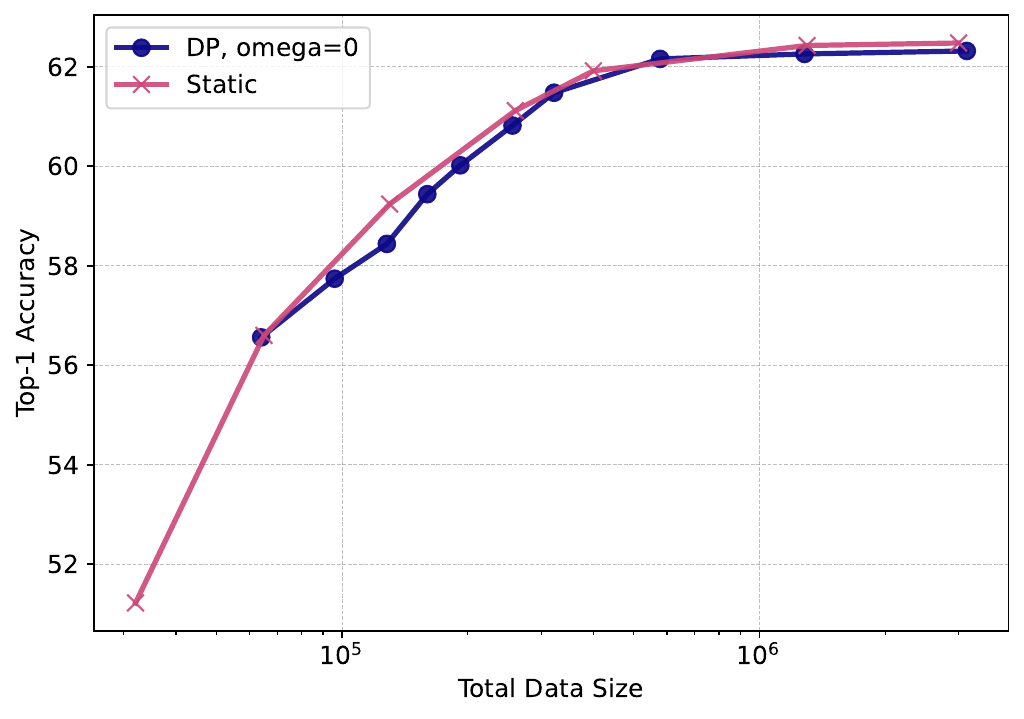}
    \caption{there is close to no difference between generating all the data in advance or generating it dynamically if we allow for enough iterations for training. Both cases have the same scaling behavior.}
    \label{fig:acc_dyn_vs_static}
\end{figure}

\subsection{Experimental Details}\label{app:exp_details}
\subsubsection{Scaling plots}\label{app:details_scaling}
We have used the Warmup-Stable-Decay (WSD) learning rate scheduler~\citep{hu2024minicpm}, which stabilizes the learning rate throughout most of the training, ensuring effective adaptation to newly generated data.
For ImageNet-100, we train on 4 nodes, each with 8 GPUs with a batchsize of 64. For ImageNet-1k, we train on 4 nodes, each with 8 GPUs with a batchsize of 128. For all the experiments, initial 10\% of the iterations is done with linear-warmup and the last 20\% of the iterations is for cool-down with Cosine Annealing. The intermediate steps are constant learning rate. For all these experiments we use $\lambda=3$ and $\omega=0.05$.

For ImageNet 100, the learning rate is $0.003$ with an EMA momentum of $0.001$. For ImageNet-1k, the learning rate is set to $0.0016$ with an EMA momentum of $0.001$. We also use label smoothing with a value of 0.11. We use Mixup with an alpha of 0.5 and CutMix with an alpha of 1.0. Furthermore, we use the AdamW optimizer. 

Furthermore, for each setup in our experiments, we apply branch-outs. A branch-out is the same experiment as an initial setup except that it does not allow additional data starting from a specific epoch. The epoch is selected based on the times that the $T_{max}$ was hit. Meaning a branch out is just before additional data is added to the training set.
\begin{table}[h]
\centering
\begin{tabular}{cccccccccc}
    \toprule
    \textbf{N} & \textbf{P} & \textbf{k} & \textbf{N + kP} & \textbf{$\omega$} & \textbf{Init. $T_{max}$} & \textbf{Branch out Epoch} & \textbf{IN Val.} & \textbf{IN-Sk} & \textbf{IN-R*} \\
    \midrule
    32000 & 16000 & 3& 80000 & 0.05 & 6 & 662 & 59.48 & 31.49 & 58.92 \\
    32000 & 16000 & 4& 96000 & 0.05 & 6 & 701 & 60.54 & 33.69 & 59.97 \\
    32000 & 16000 & 5& 112000 & 0.05 & 6 & 767 & 61.80 & 35.03 & 61.24 \\
    32000 & 16000 & 6& 128000 & 0.05 & 6 & 859 & 62.68 & 35.95 & 62.55 \\
    32000 & 16000 & 8& 160000 & 0.05 & 6 & 951 & 64.40 & 38.08 & 63.87 \\
    64000 & 32000 & 6& 256000 & 0.05 & 4 & 469 & 65.52 & 43.42 & 67.32 \\
    64000 & 32000 & 8& 320000 & 0.05 & 4 & 606 & 66.28 & 44.33 & 67.94 \\
    64000 & 32000 & 11& 416000 & 0.05 & 4 & 782 & 66.92 & 44.99 & 68.81 \\
    64000 & 32000 & 18& 640000 & 0.05 & 4 & 1001 & 67.80 & 45.25 & 68.46 \\
    130000 & 130000 & 6& 910000 & 0.05 & 14 & - & 68.28 & 45.06 & 70.87 \\
    130000 & 64000 & 27& 1794000 & 0.05 & 5 & 494 & 68.46 & 46.33 & 71.04 \\
    130000 & 64000 & 47& 3138000 & 0.05 & 5 & 618 & 68.88 & 45.76 & 71.26 \\
    64000 & 0 & - & 64000 & 0& inf & - & 56.56 & 27.86 & 52.97 \\
    130000 & 0 & - & 130000 & 0& inf & - & 59.44 & 33.32 & 55.95 \\
    260000 & 0 & - & 260000 & 0& inf & - & 60.02 & 33.79 & 56.74 \\
    400000 & 0 & - & 400000 & 0& inf & - & 61.92 & 36.03 & 59.75 \\
    2000000 & 0 & - & 2000000 & 0& inf & - & 62.16 & 34.97 & 60.15 \\
    4000000 & 0 & - & 4000000 & 0& inf & - & 62.32 & 36.43 & 60.89 \\
    \bottomrule
\end{tabular}
\caption{Details of the results reported in Figure~\ref{fig:scaling-laws} for the ImageNet-100 dataset. All the experiments are trained for 50k iterations. The variables are based on the notations defined in Algorithm~\ref{alg:framework}. Note that $T_{max}$ is incremental.}
\label{tab:your_table_label}
\end{table}

\begin{table}[h]
\centering
\begin{tabular}{cccccccccc}
    \toprule
    \textbf{N} & \textbf{P} & \textbf{k} & \textbf{N + kP} & \textbf{$\omega$} & \textbf{Init. $T_{max}$} & \textbf{Branch out Epoch} & \textbf{IN Val.} & \textbf{IN-Sk} & \textbf{IN-R} \\
    \midrule
    160000  & 160000  & 1  & 320000  & 0.05 & 1  & 134  & 42.572 & 39.363 & 20.987 \\
    320000  & 160000  & 1  & 480000  & 0.05 & 1  & 191  & 44.880 & 41.987 & 23.095 \\
    320000  & 320000  & 1  & 640000  & 0.05 & 1  & 71   & 47.910 & 46.887 & 27.568 \\
    654000  & 654000  & 1  & 1308000 & 0.05 & 1  & 124  & 50.226 & 49.867 & 29.843 \\
    654000  & 654000  & 2  & 1962000 & 0.05 & 1  & 156  & 50.670 & 51.027 & 29.944 \\
    1300000 & 650000  & 10 & 7800000 & 0.05 & 1  & 246  & 50.908 & 49.820 & 31.217 \\
    654000  & 654000  & 19 & 13080000 & 0.05 & 1  & -    & 51.198 & -      & 16.776 \\
    320000  & 0       & -  & 320000  & 0.0  & inf & -    & 39.334 & 32.653 & 18.495 \\
    654000  & 0       & -  & 654000  & 0.0  & inf & -    & 42.514 & 33.883 & 21.303 \\
    1300000 & 0       & -  & 1300000 & 0.0  & inf & -    & 44.116 & 37.337 & 23.653 \\
    2600000 & 0       & -  & 2600000 & 0.0  & inf & -    & 45.006 & 38.667 & 24.298 \\
    10000000 & 0      & -  & 10000000 & 0.0  & inf & -   & 45.614 & 40.050 & 24.762 \\
    13000000 & 0      & -  & 13000000 & 0.0  & inf & -   & 45.628 & 40.357 & - \\
    \bottomrule
\end{tabular}

\caption{Details of the results reported in Figure~\ref{fig:scaling-laws} for the ImageNet-1k dataset. All the experiments are trained for 100k iterations. The variables are based on the notations defined in Algorithm~\ref{alg:framework}. Note that $T_{max}$ is incremental.}
\label{tab:experiment_results}
\end{table}

\subsection{Visual examples}
Below we provide additional examples of generations throughout time with different $\omega$ coefficients (x-axis) of [0.0001, 0.1, 0.3, 0.5, 0.7]. All samples are generated with the same seed. As from top to bottom the epoch number increases.
\begin{figure}[h]
    \centering
    \begin{subfigure}{00.48\linewidth}
        \centering
        \includegraphics[width=\linewidth]{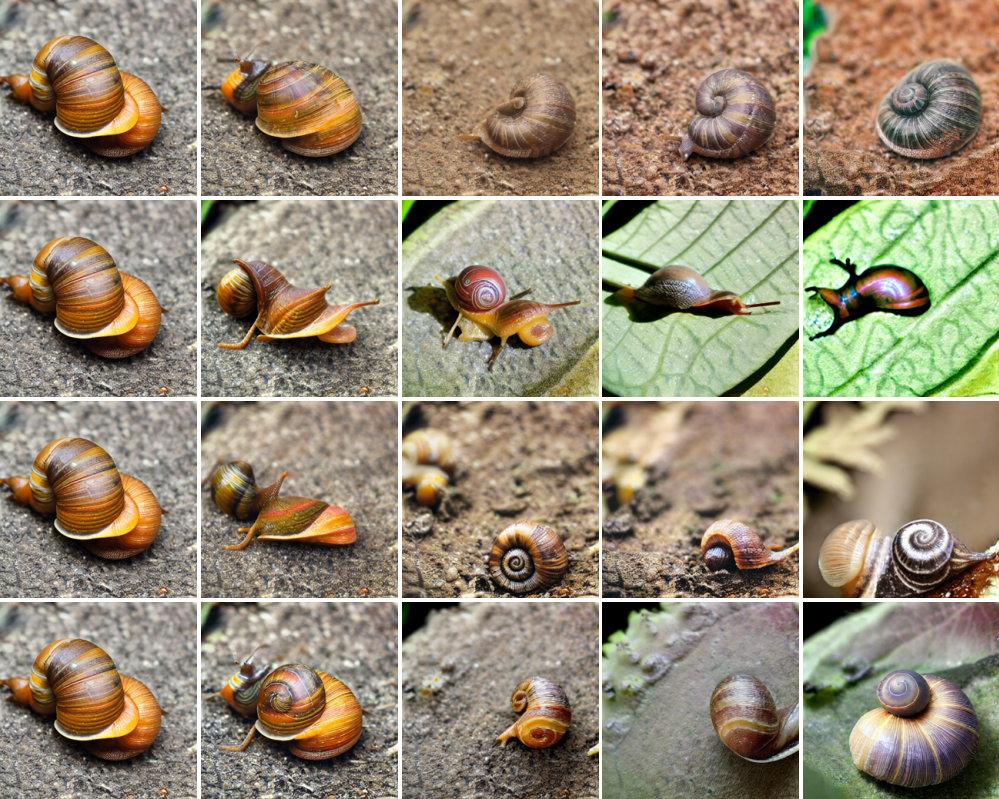}
        \caption{Class snail.}
        \label{fig:subfig1}
    \end{subfigure}
    \hfill
    \begin{subfigure}{00.48\linewidth}
        \centering
        \includegraphics[width=\linewidth]{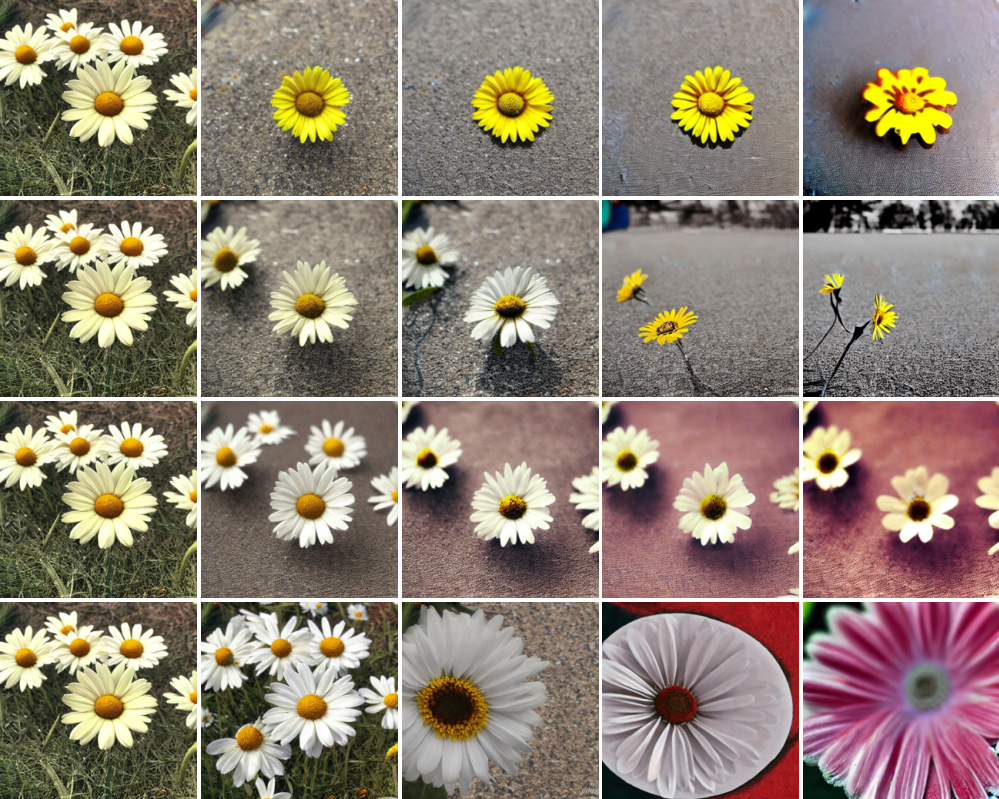}
        \caption{Class daisy.}
        \label{fig:subfig5}
    \end{subfigure}
    \vspace{0.5cm}

    \begin{subfigure}{00.48\linewidth}
        \centering
        \includegraphics[width=\linewidth]{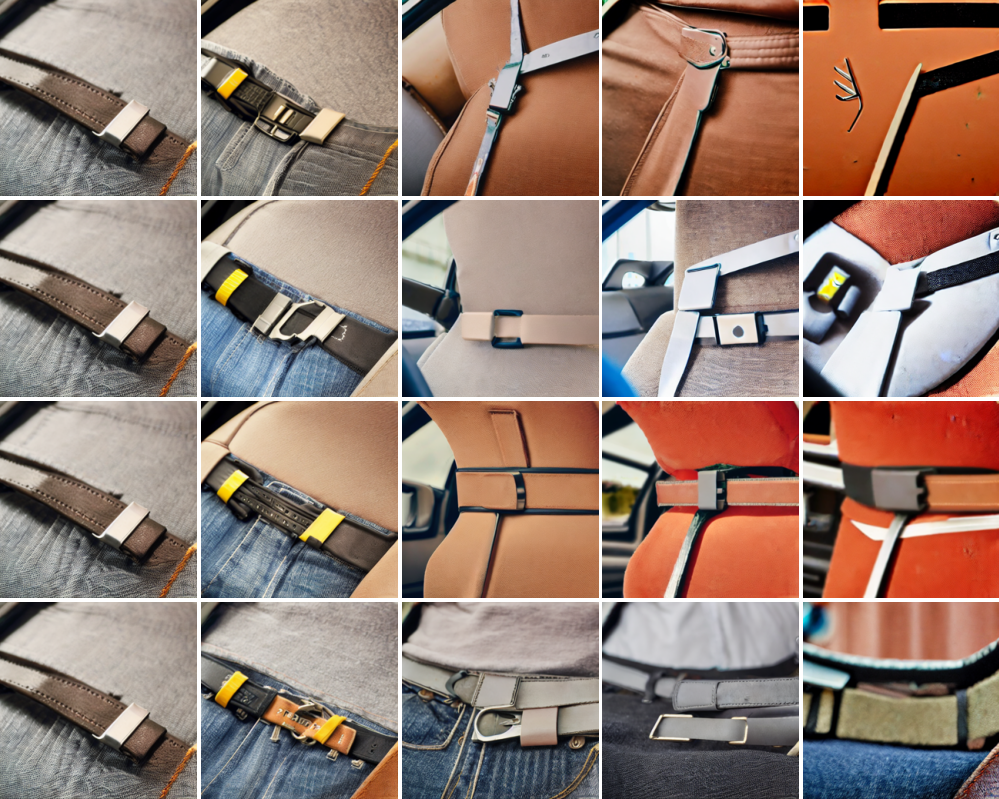}
        \caption{Class seat belt.}
        \label{fig:subfig3}
    \end{subfigure}
    \hfill
    \begin{subfigure}{00.48\linewidth}
        \centering
        \includegraphics[width=\linewidth]{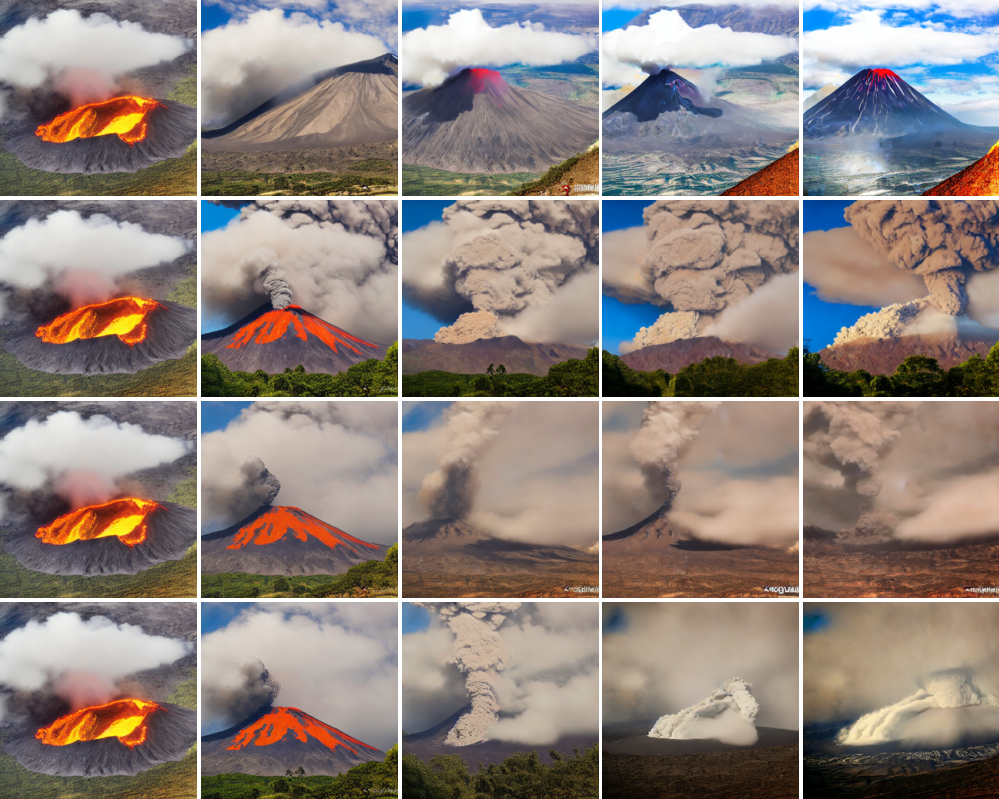}
        \caption{Class volcano.}
        \label{fig:subfig4}
    \end{subfigure}

    \vspace{0.5cm}

    \caption{Examples of generated samples for different class prompts across training epochs, with varying entropy guidance coefficient (\(\omega\)) (left to right) as the training progresses (top to bottom).}

    \label{fig:merged}
\end{figure}

\begin{figure}
    \centering
    \includegraphics[width=1\linewidth]{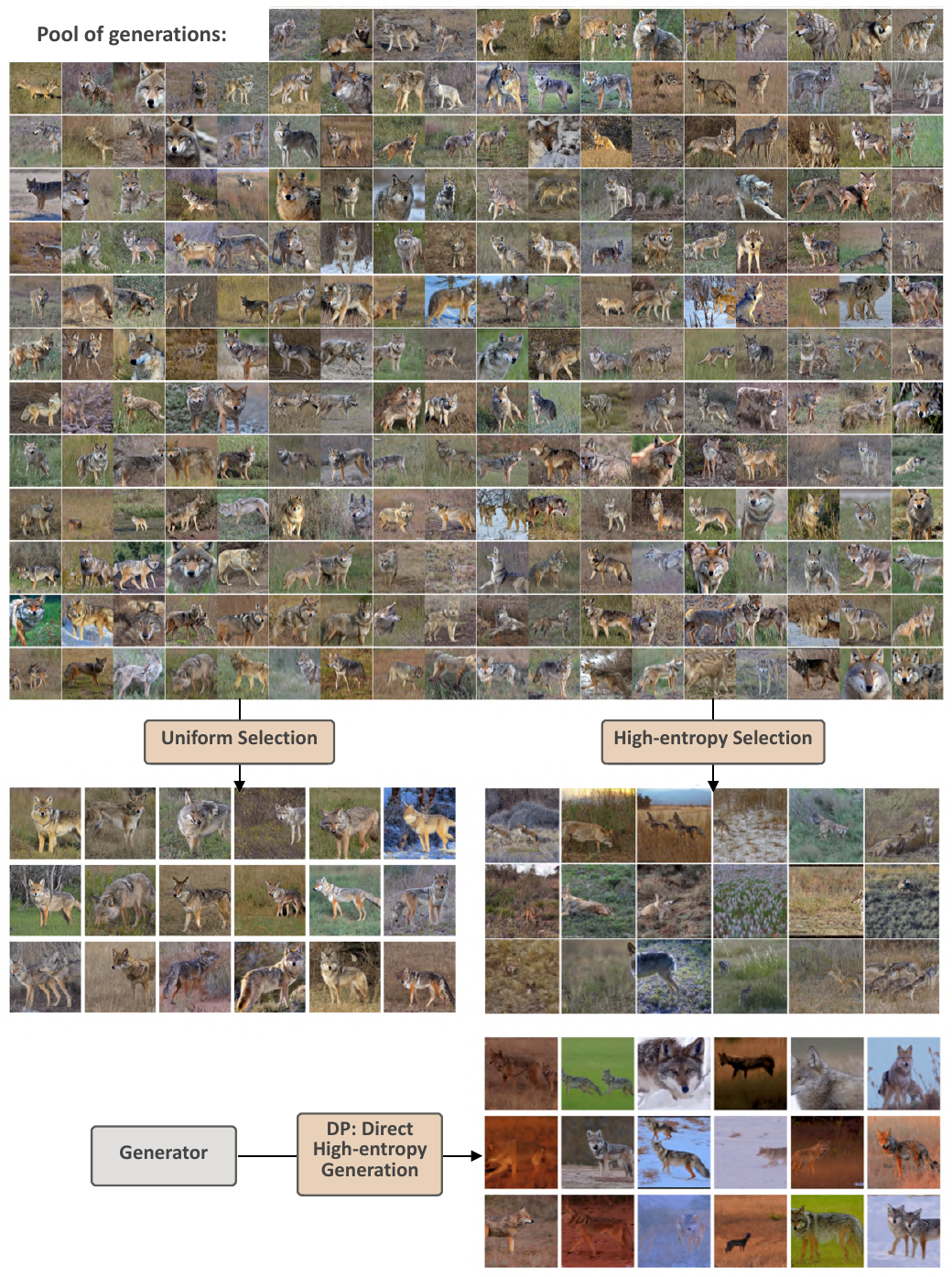}
    \caption{\textbf{Efficient and Diverse Sampling with DP:}  Instead of inefficiently over-sampling and selecting high-entropy examples, DP directly generates high-entropy samples. This not only improves computational efficiency but also results in greater visual diversity.}
    \label{fig:wolfs}
\end{figure}

\begin{figure}
    \centering
    \includegraphics[width=1\linewidth]{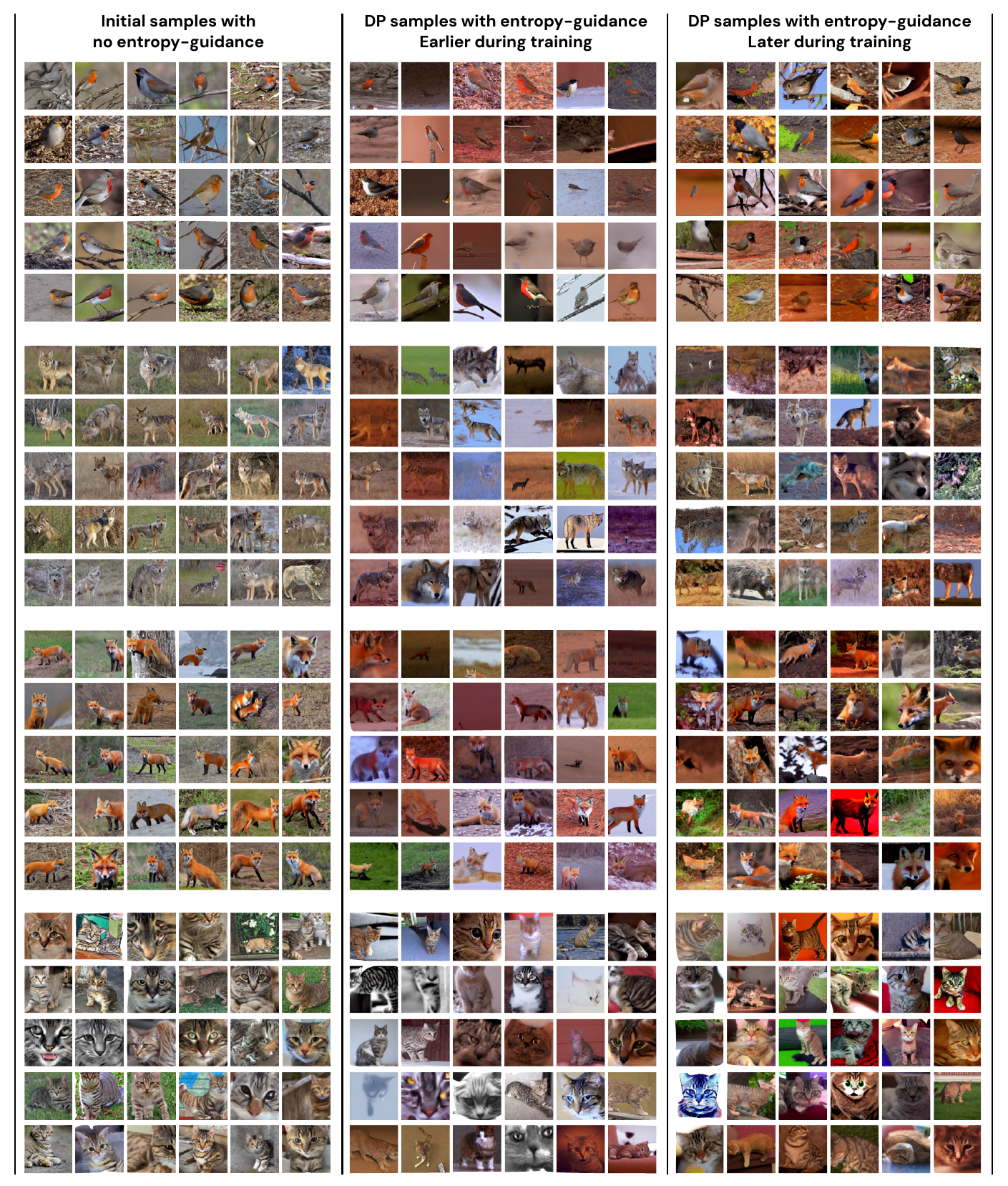}
    \caption{\textbf{Evolution of High-Entropy Samples During Training:} Early-stage generations show mainly color diversity, while later stages exhibit a richer set of transformations, aligning with the classifier's evolving uncertainties.}
    \label{fig:vis_stages}
\end{figure}

\begin{figure}
    \centering
    \includegraphics[width=0.93\linewidth]{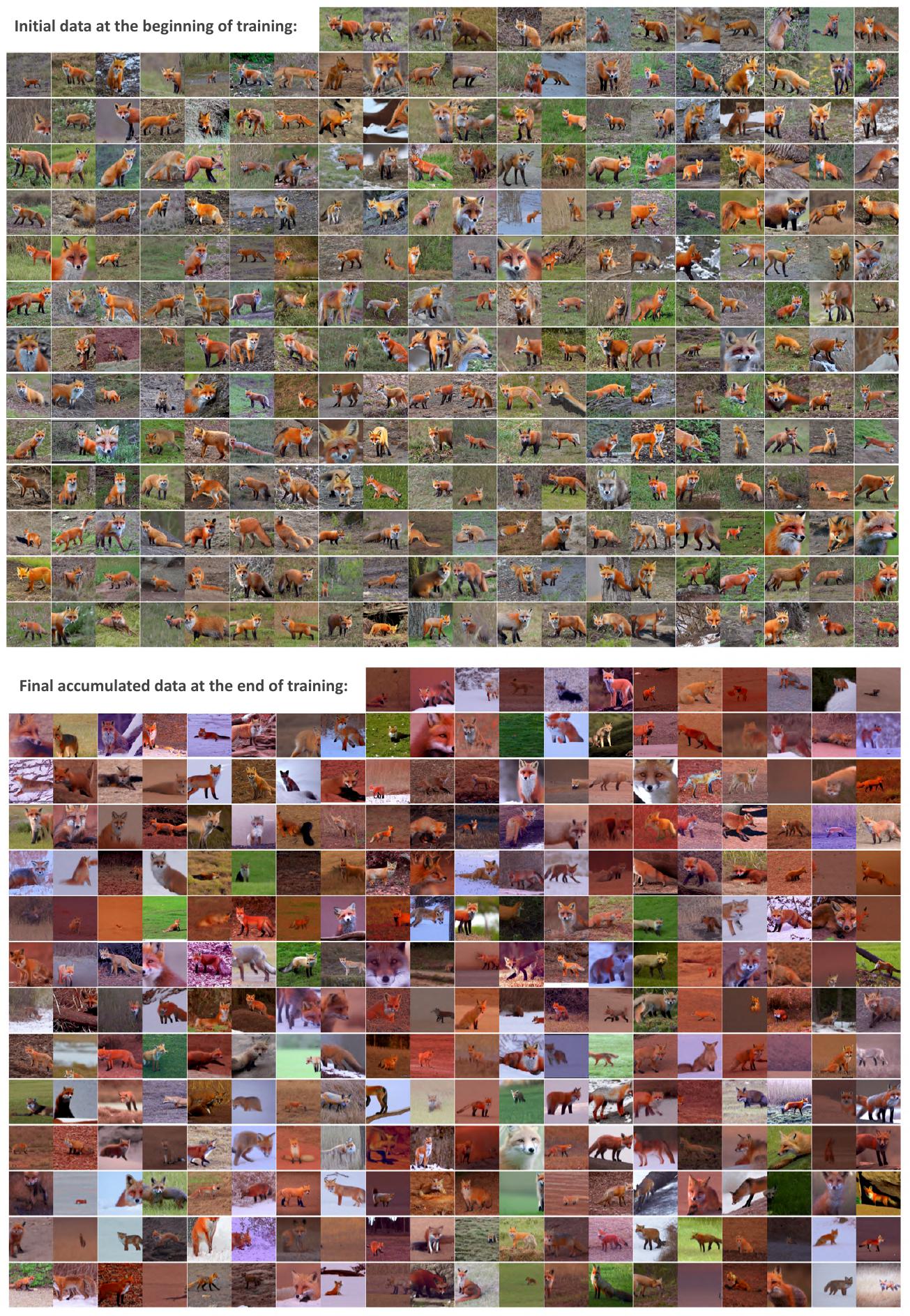}
    \caption{\textbf{Comparison of Initial and Final Training Data:} The initial training data lacks entropy guidance, as the classifier is untrained. By the end of training, the accumulated dataset contains progressively harder/diverse examples.}
    \label{fig:vis_init_vs_final}
\end{figure}

%% file: paper.bbl
\begin{thebibliography}{41}
\providecommand{\natexlab}[1]{#1}
\providecommand{\url}[1]{\texttt{#1}}
\expandafter\ifx\csname urlstyle\endcsname\relax
  \providecommand{\doi}[1]{doi: #1}\else
  \providecommand{\doi}{doi: \begingroup \urlstyle{rm}\Url}\fi

\bibitem[Astolfi et~al.(2023)Astolfi, Casanova, Verbeek, Vincent, Romero-Soriano, and Drozdzal]{astolfi2023instance}
Pietro Astolfi, Arantxa Casanova, Jakob Verbeek, Pascal Vincent, Adriana Romero-Soriano, and Michal Drozdzal.
\newblock Instance-conditioned gan data augmentation for representation learning.
\newblock \emph{arXiv preprint arXiv:2303.09677}, 2023.

\bibitem[Astolfi et~al.(2024)Astolfi, Careil, Hall, Ma{\~n}as, Muckley, Verbeek, Soriano, and Drozdzal]{astolfi2024consistency}
Pietro Astolfi, Marlene Careil, Melissa Hall, Oscar Ma{\~n}as, Matthew Muckley, Jakob Verbeek, Adriana~Romero Soriano, and Michal Drozdzal.
\newblock Consistency-diversity-realism pareto fronts of conditional image generative models.
\newblock \emph{arXiv preprint arXiv:2406.10429}, 2024.

\bibitem[Bang et~al.(2024)Bang, Ahn, and Lee]{bang2024active}
Jihwan Bang, Sumyeong Ahn, and Jae-Gil Lee.
\newblock Active prompt learning in vision language models.
\newblock In \emph{CVPR}, 2024.

\bibitem[Couillet and Liao(2022)]{Couillet_Liao_2022}
Romain Couillet and Zhenyu Liao.
\newblock \emph{Random Matrix Methods for Machine Learning}.
\newblock Cambridge University Press, 2022.

\bibitem[Deng et~al.(2009)Deng, Dong, Socher, Li, Li, and Fei-Fei]{deng2009imagenet}
Jia Deng, Wei Dong, Richard Socher, Li-Jia Li, Kai Li, and Li~Fei-Fei.
\newblock Imagenet: A large-scale hierarchical image database.
\newblock In \emph{2009 IEEE conference on computer vision and pattern recognition}, pages 248--255. Ieee, 2009.

\bibitem[Dohmatob et~al.(2024{\natexlab{a}})Dohmatob, Feng, Subramonian, and Kempe]{dohmatob2024strong}
Elvis Dohmatob, Yunzhen Feng, Arjun Subramonian, and Julia Kempe.
\newblock Strong model collapse.
\newblock \emph{arXiv preprint arXiv:2410.04840}, 2024{\natexlab{a}}.

\bibitem[Dohmatob et~al.(2024{\natexlab{b}})Dohmatob, Feng, Yang, Charton, and Kempe]{dohmatob2024tale}
Elvis Dohmatob, Yunzhen Feng, Pu~Yang, Francois Charton, and Julia Kempe.
\newblock A tale of tails: Model collapse as a change of scaling laws.
\newblock \emph{arXiv preprint arXiv:2402.07043}, 2024{\natexlab{b}}.

\bibitem[Dosovitskiy et~al.(2021)Dosovitskiy, Beyer, Kolesnikov, Weissenborn, Zhai, Unterthiner, Dehghani, Minderer, Heigold, Gelly, Uszkoreit, and Houlsby]{dosovitskiy2020image}
Alexey Dosovitskiy, Lucas Beyer, Alexander Kolesnikov, Dirk Weissenborn, Xiaohua Zhai, Thomas Unterthiner, Mostafa Dehghani, Matthias Minderer, Georg Heigold, Sylvain Gelly, Jakob Uszkoreit, and Neil Houlsby.
\newblock An image is worth 16$\times$16 words: Transformers for image recognition at scale.
\newblock In \emph{ICLR}, 2021.

\bibitem[Dunlap et~al.(2023)Dunlap, Umino, Zhang, Yang, Gonzalez, and Darrell]{dunlap2023diversify}
Lisa Dunlap, Alyssa Umino, Han Zhang, Jiezhi Yang, Joseph~E Gonzalez, and Trevor Darrell.
\newblock Diversify your vision datasets with automatic diffusion-based augmentation.
\newblock In \emph{NeurIPS}, 2023.

\bibitem[Ericsson et~al.(1993)Ericsson, Krampe, and Tesch-R{\"o}mer]{ericsson1993role}
K~Anders Ericsson, Ralf~T Krampe, and Clemens Tesch-R{\"o}mer.
\newblock The role of deliberate practice in the acquisition of expert performance.
\newblock \emph{Psychological review}, 100\penalty0 (3):\penalty0 363, 1993.

\bibitem[Evans et~al.(2023)Evans, Pathak, Merzic, Schwarz, Tanno, and Henaff]{evans2023bad}
Talfan Evans, Shreya Pathak, Hamza Merzic, Jonathan Schwarz, Ryutaro Tanno, and Olivier~J Henaff.
\newblock Bad students make great teachers: Active learning accelerates large-scale visual understanding.
\newblock \emph{arXiv preprint}, 2312.05328, 2023.

\bibitem[Fan et~al.(2024)Fan, Chen, Krishnan, Katabi, Isola, and Tian]{fan2024scaling}
Lijie Fan, Kaifeng Chen, Dilip Krishnan, Dina Katabi, Phillip Isola, and Yonglong Tian.
\newblock Scaling laws of synthetic images for model training... for now.
\newblock In \emph{CVPR}, 2024.

\bibitem[Feng et~al.(2024)Feng, Dohmatob, Yang, Charton, and Kempe]{feng2024modelcollapsescalingsynthesized}
Yunzhen Feng, Elvis Dohmatob, Pu~Yang, Francois Charton, and Julia Kempe.
\newblock Beyond model collapse: Scaling up with synthesized data requires reinforcement, 2024.
\newblock \url{https://arxiv.org/abs/2406.07515}.

\bibitem[Firdoussi et~al.(2024)Firdoussi, Seddik, Hayou, Alami, Alzubaidi, and Hacid]{Firdoussi2024}
Aymane~El Firdoussi, Mohamed El~Amine Seddik, Soufiane Hayou, Reda Alami, Ahmed Alzubaidi, and Hakim Hacid.
\newblock Maximizing the potential of synthetic data: Insights from random matrix theory, 2024.

\bibitem[Grattafiori et~al.(2024)]{grattafiori2024llama3herdmodels}
Aaron Grattafiori et~al.
\newblock The llama 3 herd of models, 2024.
\newblock \url{https://arxiv.org/abs/2407.21783}.

\bibitem[Hemmat et~al.(2023)Hemmat, Pezeshki, Bordes, Drozdzal, and Romero-Soriano]{hemmat2023feedback}
Reyhane~Askari Hemmat, Mohammad Pezeshki, Florian Bordes, Michal Drozdzal, and Adriana Romero-Soriano.
\newblock Feedback-guided data synthesis for imbalanced classification.
\newblock \emph{arXiv preprint}, 2310.00158, 2023.

\bibitem[Hendrycks et~al.(2021{\natexlab{a}})Hendrycks, Basart, Mu, Kadavath, Wang, Dorundo, Desai, Zhu, Parajuli, Guo, et~al.]{hendrycks2021many}
Dan Hendrycks, Steven Basart, Norman Mu, Saurav Kadavath, Frank Wang, Evan Dorundo, Rahul Desai, Tyler Zhu, Samyak Parajuli, Mike Guo, et~al.
\newblock The many faces of robustness: A critical analysis of out-of-distribution generalization.
\newblock In \emph{Proceedings of the IEEE/CVF international conference on computer vision}, pages 8340--8349, 2021{\natexlab{a}}.

\bibitem[Hendrycks et~al.(2021{\natexlab{b}})Hendrycks, Zhao, Basart, Steinhardt, and Song]{hendrycks2021natural}
Dan Hendrycks, Kevin Zhao, Steven Basart, Jacob Steinhardt, and Dawn Song.
\newblock Natural adversarial examples.
\newblock In \emph{Proceedings of the IEEE/CVF conference on computer vision and pattern recognition}, pages 15262--15271, 2021{\natexlab{b}}.

\bibitem[Ho and Salimans(2022)]{ho2022classifier}
Jonathan Ho and Tim Salimans.
\newblock Classifier-free diffusion guidance.
\newblock \emph{arXiv preprint arXiv:2207.12598}, 2022.

\bibitem[Hu et~al.(2024)Hu, Tu, Han, He, Cui, Long, Zheng, Fang, Huang, Zhao, et~al.]{hu2024minicpm}
Shengding Hu, Yuge Tu, Xu~Han, Chaoqun He, Ganqu Cui, Xiang Long, Zhi Zheng, Yewei Fang, Yuxiang Huang, Weilin Zhao, et~al.
\newblock Minicpm: Unveiling the potential of small language models with scalable training strategies.
\newblock \emph{arXiv preprint}, 2404.06395, 2024.

\bibitem[Kirkpatrick et~al.(2017)Kirkpatrick, Pascanu, Rabinowitz, Veness, Desjardins, Rusu, Milan, Quan, Ramalho, Grabska-Barwinska, et~al.]{kirkpatrick2017overcoming}
James Kirkpatrick, Razvan Pascanu, Neil Rabinowitz, Joel Veness, Guillaume Desjardins, Andrei~A Rusu, Kieran Milan, John Quan, Tiago Ramalho, Agnieszka Grabska-Barwinska, et~al.
\newblock Overcoming catastrophic forgetting in neural networks.
\newblock \emph{Proceedings of the national academy of sciences}, 114\penalty0 (13):\penalty0 3521--3526, 2017.

\bibitem[Kolossov et~al.(2024)Kolossov, Montanari, and Tandon]{kolossov2024towards}
Germain Kolossov, Andrea Montanari, and Pulkit Tandon.
\newblock Towards a statistical theory of data selection under weak supervision.
\newblock In \emph{The Twelfth International Conference on Learning Representations}, 2024.
\newblock \url{https://openreview.net/forum?id=HhfcNgQn6p}.

\bibitem[Li et~al.(2025)Li, Yang, Li, Wu, Yu, Ghanem, and Zhang]{li2025genview}
Xiaojie Li, Yibo Yang, Xiangtai Li, Jianlong Wu, Yue Yu, Bernard Ghanem, and Min Zhang.
\newblock Genview: Enhancing view quality with pretrained generative model for self-supervised learning.
\newblock In \emph{European Conference on Computer Vision}, pages 306--325. Springer, 2025.

\bibitem[Liao and Mahoney(2021)]{ZhenyuAndMahoney2021}
Zhenyu Liao and Michael~W Mahoney.
\newblock Hessian eigenspectra of more realistic nonlinear models.
\newblock In \emph{Advances in Neural Information Processing Systems}, volume~34. Curran Associates, Inc., 2021.

\bibitem[Mar{\v{c}}enko and Pastur(1967)]{MP1967}
V~A Mar{\v{c}}enko and L~A Pastur.
\newblock Distribution of eigenvalues for some sets of random matrices.
\newblock \emph{Mathematics of the USSR-Sbornik}, 1\penalty0 (4):\penalty0 457, apr 1967.

\bibitem[Oksendal(2013)]{oksendal2013stochastic}
Bernt Oksendal.
\newblock \emph{Stochastic differential equations: an introduction with applications}.
\newblock Springer Science \& Business Media, 2013.

\bibitem[Recht et~al.(2019)Recht, Roelofs, Schmidt, and Shankar]{recht2019imagenet}
Benjamin Recht, Rebecca Roelofs, Ludwig Schmidt, and Vaishaal Shankar.
\newblock Do imagenet classifiers generalize to imagenet?
\newblock In \emph{International conference on machine learning}, pages 5389--5400. PMLR, 2019.

\bibitem[Rombach et~al.(2022)Rombach, Blattmann, Lorenz, Esser, and Ommer]{rombach2022high}
Robin Rombach, Andreas Blattmann, Dominik Lorenz, Patrick Esser, and Bj{\"o}rn Ommer.
\newblock High-resolution image synthesis with latent diffusion models.
\newblock In \emph{CVPR}, 2022.

\bibitem[Sar{\i}y{\i}ld{\i}z et~al.(2023)Sar{\i}y{\i}ld{\i}z, Alahari, Larlus, and Kalantidis]{sariyildiz2023fake}
Mert~B{\"u}lent Sar{\i}y{\i}ld{\i}z, Karteek Alahari, Diane Larlus, and Yannis Kalantidis.
\newblock Fake it till you make it: Learning transferable representations from synthetic imagenet clones.
\newblock In \emph{CVPR}, 2023.

\bibitem[Settles(2009)]{settles2009active}
Burr Settles.
\newblock Active learning literature survey.
\newblock 2009.

\bibitem[Shin et~al.(2023)Shin, Kang, and Park]{shin2023fill}
Joonghyuk Shin, Minguk Kang, and Jaesik Park.
\newblock Fill-up: Balancing long-tailed data with generative models.
\newblock \emph{arXiv preprint}, 2306.07200, 2023.

\bibitem[Song et~al.(2020)Song, Meng, and Ermon]{song2020denoising}
Jiaming Song, Chenlin Meng, and Stefano Ermon.
\newblock Denoising diffusion implicit models.
\newblock \emph{arXiv preprint}, 2010.02502, 2020.

\bibitem[Song and Ermon(2019)]{song2019generative}
Yang Song and Stefano Ermon.
\newblock Generative modeling by estimating gradients of the data distribution.
\newblock \emph{Advances in neural information processing systems}, 32, 2019.

\bibitem[Sorscher et~al.(2022)Sorscher, Geirhos, Shekhar, Ganguli, and Morcos]{sorscher2022beyond}
Ben Sorscher, Robert Geirhos, Shashank Shekhar, Surya Ganguli, and Ari Morcos.
\newblock Beyond neural scaling laws: beating power law scaling via data pruning.
\newblock \emph{Advances in Neural Information Processing Systems}, 35:\penalty0 19523--19536, 2022.

\bibitem[Tian et~al.(2020)Tian, Krishnan, and Isola]{tian2020contrastive}
Yonglong Tian, Dilip Krishnan, and Phillip Isola.
\newblock Contrastive multiview coding.
\newblock In \emph{Computer Vision--ECCV 2020: 16th European Conference, Glasgow, UK, August 23--28, 2020, Proceedings, Part XI 16}, pages 776--794. Springer, 2020.

\bibitem[Tian et~al.(2024{\natexlab{a}})Tian, Fan, Chen, Katabi, Krishnan, and Isola]{tian2024learning}
Yonglong Tian, Lijie Fan, Kaifeng Chen, Dina Katabi, Dilip Krishnan, and Phillip Isola.
\newblock Learning vision from models rivals learning vision from data.
\newblock In \emph{CVPR}, 2024{\natexlab{a}}.

\bibitem[Tian et~al.(2024{\natexlab{b}})Tian, Fan, Isola, Chang, and Krishnan]{tian2024stablerep}
Yonglong Tian, Lijie Fan, Phillip Isola, Huiwen Chang, and Dilip Krishnan.
\newblock Stablerep: Synthetic images from text-to-image models make strong visual representation learners.
\newblock In \emph{NeurIPS}, 2024{\natexlab{b}}.

\bibitem[Wang et~al.(2019)Wang, Ge, Lipton, and Xing]{wang2019learning}
Haohan Wang, Songwei Ge, Zachary Lipton, and Eric~P Xing.
\newblock Learning robust global representations by penalizing local predictive power.
\newblock \emph{Advances in Neural Information Processing Systems}, 32, 2019.

\bibitem[Yuan et~al.(2024)Yuan, Chen, Ji, and Gu]{yuan2024self}
Huizhuo Yuan, Zixiang Chen, Kaixuan Ji, and Quanquan Gu.
\newblock Self-play fine-tuning of diffusion models for text-to-image generation.
\newblock \emph{arXiv preprint}, 2402.10210, 2024.

\bibitem[Yuan et~al.(2023)Yuan, Zhang, Sun, Torr, and Zhao]{yuan2023real}
Jianhao Yuan, Jie Zhang, Shuyang Sun, Philip Torr, and Bo~Zhao.
\newblock Real-fake: Effective training data synthesis through distribution matching.
\newblock \emph{arXiv preprint}, 2310.10402, 2023.

\bibitem[Zhang et~al.(2024)Zhang, Zhou, Hooi, Wang, and Feng]{zhang2024expanding}
Yifan Zhang, Daquan Zhou, Bryan Hooi, Kai Wang, and Jiashi Feng.
\newblock Expanding small-scale datasets with guided imagination.
\newblock In \emph{NeurIPS}, 2024.

\end{thebibliography}
